\documentclass[11pt]{article}

\usepackage[utf8]{inputenc} 
\usepackage[T1]{fontenc}    
\usepackage{hyperref}       
\usepackage{url}            
\usepackage{booktabs}       
\usepackage{amsfonts}       
\usepackage{nicefrac}       
\usepackage{microtype}      

\usepackage{times}
\usepackage{fullpage}

\usepackage{amsmath}
\usepackage{amssymb}
\usepackage{bm}
\usepackage{dsfont}
\usepackage{mathtools}
\usepackage{wrapfig}

\usepackage{algpseudocode}
\usepackage{algorithm}
\usepackage{cleveref}
\usepackage{amsthm}
\usepackage{graphicx}
\usepackage{amsmath}
\usepackage[version=4]{mhchem}
\usepackage{siunitx}
\usepackage{longtable,tabularx}
\newtheoremstyle{mystyle}
  {}
  {}
  {\itshape}
  {}
  {\bfseries}
  {.}
  { }
  {\thmname{#1}\thmnumber{ #2}\thmnote{ (#3)}}

\theoremstyle{mystyle}
\newtheorem{theorem}{Theorem}
\newtheorem{remark}{Remark}
\newtheorem{prop}{Proposition}
\newtheorem{cor}{Corollary}
\newtheorem{lem}{Lemma}
\newtheorem{definition}{Definition}

\usepackage{cancel}

\renewcommand{\epsilon}{{\varepsilon}}

\usepackage{subcaption}
\usepackage{dsfont}

\usepackage[most]{tcolorbox}
\tcbset{colframe=blue!40!black,colback=white,colupper=black!100!black,fonttitle=\bfseries,nobeforeafter,center title, opacityback=100}

\usepackage{color}

\title{Double Descent Risk and Volume Saturation Effects:\\A Geometric Perspective}

\author{%
  Prasad Cheema \\
  The University of Sydney\\
  \texttt{prasad.cheema@sydney.edu.au}
  \and
  Mahito Sugiyama \\
  National Institute of Informatics\\
  JST, PRESTO\\
  \texttt{mahito@nii.ac.jp}
}

\date{}

\begin{document}

\maketitle
\begin{abstract}
The appearance of the double-descent risk phenomenon has received growing interest in the machine learning and statistics community, as it challenges well-understood notions behind the U-shaped train-test curves. Motivated through Rissanen's minimum description length (MDL), Balasubramanian's Occam's Razor, and Amari's information geometry, we investigate how the logarithm of the model volume: $\log V$, works to extend intuition behind the AIC and BIC model selection criteria. We find that for the particular model classes of isotropic linear regression and statistical lattices, the $\log V$ term may be decomposed into a sum of distinct components, each of which assist in their explanations of the appearance of this phenomenon. In particular they suggest why generalization error does not necessarily continue to grow with increasing model dimensionality.
\end{abstract}

\section{Introduction}
\begin{wrapfigure}{r}[0pt]{.35\textwidth}
  \centering
 \includegraphics[width=.35\textwidth]{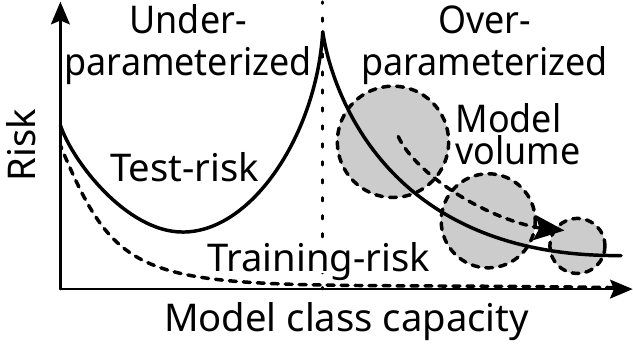}
 \caption{Double Descent Risk.}
 \label{fig:Intro_Belkin}
\end{wrapfigure}
Model selection is a problem which underpins the field of machine learning. Key to its formulation is the notion of learning an appropriate predictor, $h^{\star}: \mathbb{R}^D \rightarrow \mathbb{R}$ from an underlying model class, $\mathcal{H}$, based on $N$ input training examples $\{(\bm{x}_i,y_i)\}_{i=1}^N$, with each $(\bm{x}_i,y_i)\in \mathbb{R}^D\times \mathbb{R}$. Typically, the predictor, $h^{\star}$, is chosen so as to minimize some risk functional; that is, $h^{\star} = \arg\min_{h\in\mathcal{H}} R(h)$ with $R(h) = \mathbb{E}_{p(x,y)}[L(h(x),y)]$, where $L : \mathbb{R}\times\mathbb{R} \rightarrow \mathbb{R}$ is the risk functional, and $p(x,y)$ denotes the probability density function (pdf) over the data. Fundamentally, the aim of such an approach is to ensure that $h^{\star}$ provides good generalization capability, so that after training it minimizes the \textit{out-of-sample} test error \cite{bishop2006pattern}. This is historically estimated via the Akaike information criterion (AIC) \cite{akaike1973information}, the Bayesian information criterion (BIC) \cite{schwarz1978estimating}, or through cross validation \cite{bishop2006pattern}. AIC and BIC are derived based on asymptotic assumptions in the sample size $N$, and work similarly. Moreover, both criteria suggest that out-of-sample error increases as $\mathcal{O}(D)$\footnote{Technically AIC has a $2D$ model complexity term, and BIC has a $D\log N$ model complexity term. We will refer to the implied effects of both model complexities simply as the ``$\mathcal{O}(D)$ terms''.}, suggesting that an over-parameterized model should generalize poorly, which is an idea consistent with traditional empirical evidence, via the U-shaped \textit{train-test} curves \cite{bishop2006pattern}.

Recently, however, particular classes of highly parameterized models such as deep neural networks, and random forests have been shown to generalize extremely well, working in contrary to the implied $\mathcal{O}(D)$ model complexity effects. In fact, strong empirical evidence has been presented by Belkin et al.~\cite{belkin2019reconciling}, where it was shown that a \textit{double descent risk} phenomenon may be observed for a variety of models which transition into highly parameterized regimes. This phenomenon is shown in Figure \ref{fig:Intro_Belkin}, with many additional experiments made clear in \cite{nakkiran2019deep}. In an effort to explain such trends Belkin et al.~\cite{belkin2018understand} have tried to infer some similarities between ReLU networks and traditional kernel models, and Geiger et al. \cite{geiger2020scaling} have tried to connect the double descent cusp-like behaviour with diverging norms, through a neural tangent kernel framework. In addition, double descent risk has been explored in a variety  of simpler (and shallow) model classes \cite{ghorbani2019linearized,ba2020generalization,d2020double}, with various risk asymptotics established \cite{bartlett2019benign,hastie2019surprises}. Lastly and rather interestingly, it has been found that there exist certain parallels between double descent risk, and the notion of the \textit{jamming transition} which occurs in physical materials which undergo a phase transition~\cite{spigler2018jamming,geiger2019jamming}. 

In this paper, we attempt to \emph{understand the double descent phenomenon from a geometrically driven perspective}. This will be achieved via the notion of a \textit{model volume}, and shown to have interpretations which stem from information theory (coding theory, and signal analysis), as well as Occam's razor. We find that for a variety of simple statistical models, the model volume can decrease with increasing $D$, which ultimately serves to overpower the magnitude of the lower order model complexity term: $\mathcal{O}(D)$. This idea is also clarified in Figure \ref{fig:Intro_Belkin}.

\subsection{Model Selection and Occam's Razor}

In the late 90s and early 2000s, extensions to the base AIC and BIC formulations were developed by Rissanen~\cite{rissanen1997stochastic} and Balasubramanian~\cite{balasubramanian2005mdl}, which include additional model-specific terms. From the perspective of coding theory, Rissanen developed a notion of \textit{stochastic model complexity}, which builds upon Shannon's information criteria used for lossless encoding \cite{shannon1948mathematical}. Upon this notion, Rissanen formalized an extension of binary Shannon entropy to continuous function classes, via the discretization of the model manifold over approximately equivalent model classes. This approach establishes an intuition behind ``model distinguishability'', which is also echoed by Balasubramanian. In particular, under Risannen's construction information is encoded in \textit{nats} (as opposed to bits) and it is formally recognized as the \textit{Minimum Description Length} (MDL). This is shown in the following equation:
\begin{equation} \label{eqn:riss}
    -\log(p(\bm{x})) = \overbrace{-\log(\hat{\mathcal{L}})  + \frac{D}{2}\log \left(\frac{N}{2\pi e} \right)}^{\text{\tiny{AIC / BIC - like term}}}+ \overbrace{\log \int_{\Theta} \sqrt{\det\left(\mathcal{I}(\bm{\theta})\right)}d\bm{\theta}}^{\text{\tiny{Log - Model Volume}}} + {o}(1),
\end{equation}
where $\bm{x}=\{\bm{x}_i\}_{i=1}^N$ denotes a random vector of $N$ data samples,  $\hat{\mathcal{L}}\in\mathbb{R}$ is the likelihood function evaluated at its optimal parameter setting, with $\Theta$ being the space of possible parameter settings, and $\mathcal{I}(\bm{\theta})\in\mathbb{R}^{D\times D}$ denotes the Fisher information matrix (FIM), which is traditionally used as a lower bound on the variance of unbiased estimators, and in Jeffrey's prior~\cite{bishop2006pattern}. In a parallel fashion, Balasubramanian approached the problem of model selection, albeit from a more Bayesian perspective, which curiously yielded the same equation as in Rissanen's MDL (with the exception of one additional term, and ignoring constants)~\cite{balasubramanian1996geometric}. To achieve this he takes an alternate route based on specifying a Jeffrey's prior over the underlying parameter space, which is noted to act as an appropriate measure for the density of distinguishable distributions \cite{balasubramanian2005mdl}. Ultimately, in both Rissanen's and Balasubramanian's model selection criteria, there is a term which acts like the $\mathcal{O}(D)$ model complexity used in AIC and BIC, and an additional term: $\log  \int_{\Theta} \sqrt{\det\left(\mathcal{I}(\bm{\theta})\right)}d\bm{\theta} $, which they collectively referred to as a model distinguishability-like term. Moreover, since these methods are built upon the log-marginal: $-\log(p(\bm{x}))$, they also appeal to a Bayesian Occam's razor-like principle. Ultimately, this paper aims to explore this additional term: $\log \int_{\Theta} \sqrt{\det\left(\mathcal{I}(\bm{\theta})\right)}d\bm{\theta}$, through a geometric lens. Moreover, it will be made clear that this term also describes the underlying log-model volume, which we denote by $\log V$ \cite{amari2016information}. Ultimately, it will be shown that, $\log V$ can in fact \emph{decrease} with increasing $D$ for several simple model classes, thus serving to counter-act the behaviours of the $\mathcal{O}(D)$ model complexity term. By implication, this suggests that certain model classes have the power to generalize well when transitioning into over-parameterized regimes.

\subsection{Information Geometry}\label{sec:info_geo}

Information geometry concerns the application of differential geometry to statistical models. In particular, it considers a statistical manifold, $\mathcal{M}=\{p(x;\theta)\}$, over a $\theta$ co-ordinate system. A Riemannian metric, $\mathcal{G}$, can be placed on $\mathcal{M}$, where $\mathcal{G}:  \mathcal{T}_p(\mathcal{M}) \times \mathcal{T}_p(\mathcal{M}) \rightarrow \mathbb{R}_{\geq 0}$ for each point $p\in\mathcal{M}$, with $\mathcal{T}_p(\mathcal{M})$ defined as the local tangent space at point $p$ on the manifold. Principally, $\mathcal{G}$ is a generalization of the inner product on Euclidean spaces to Riemannian manifolds. In addition, Amari defines a dually coupled affine co-ordinate system on statistical manifolds. Dually coupled co-ordinates arise naturally from the \textit{dually flat} property which is intrinsic to many information manifolds. These co-ordinates are known as the $\theta$ (e-flat) and $\eta$ (m-flat) co-ordinates for the exponential family in particular, and are related through the Legendre transformation $\eta = \nabla \psi (\theta)$ and $\theta = \nabla \varphi (\eta)$ via two convex functions $\psi, \varphi: \mathbb{R}^D \rightarrow \mathbb{R}$~\cite{amari2007methods}. The $\theta$ and $\eta$ co-ordinates for exponential models correspond to the natural and expectation parameters, respectively. Furthermore, the FIM defines a natural Riemannian metric tensor: $\mathcal{G}_{ij} = \mathbb{E}[\partial_i \log(p(x;\theta)) \partial_j \log(p(x;\theta))] = \mathcal{I}_{ij}$~\cite{rao1992information,amari2007methods}.
Thus the motivation for using information geometry is clear as Rissanen's MDL, and Balasubramanian's Occam Razor, depend on the FIM, which is, geometrically speaking, the metric tensor. Consequently, the term $\log \int_{\Theta} \sqrt{\det\left(\mathcal{I}(\bm{\theta})\right)}d\bm{\theta}$ has a clear definition in differential geometry as being the log-\textit{volume} of the underlying information manifold; that is, the square root of the determinant of the Fisher information matrix is the manifold volume~\cite{amari2007methods,jeffreys1946invariant}. Unfortunately, the FIM is singular ($\text{rank}(\mathcal{I}(\theta)) < D$) for many modern statistical models, implying that classical assumptions such as the asymptotic normality of the maximum likelihood estimator, does not hold, and that the BIC is not necessarily equal to the Bayes marginal likelihood \cite{watanabe2009algebraic}. The implication of this will be explored for the studied models.


\section{Double Descent Risk and the Volume of Statistical Models}
In the following subsections we will explore how the notion of model volume suggests that in some classic statistical models, increasing $D$ can result in a decreasing, or even saturating volume term, which can act to counter-balance the poor generalization properties traditionally implied by $\mathcal{O}(D)$ model complexity. In particular, this will be demonstrated for the popular machine learning models of: (1) Isotropic linear regression, and (2) Statistical lattice models (of which Boltzmann machines form a subset). Moreover, a brief discussion on how this phenomena arises in the classic stochastic perceptron unit is provided in Appendix \ref{app:sec:perceptron_title}.

\subsection{Isotropic Linear Regression}




We define isotropic linear regression as the problem: $\bm{y} = \bm{X}\bm{\beta} + \bm{\varepsilon}$, for a given dataset $\{(\bm{x}_i,y_i)\}_{i=1}^N$, where each $(\bm{x}_i ,y_i)\in\mathbb{R}^D\times\mathbb{R}$, $\bm{\varepsilon}_i \sim \mathcal{N}(0,\sigma^2)$ and each $\bm{x}_{ij}\sim \mathcal{N}(0,1)$. Furthermore, we place a hard power constraint on $\bm{\beta}\in\mathbb{R}^D$ which is to be interpreted as $\|\bm{\beta}\|_2^2 \leq P$. A relationship between this model and the notion of channel capacity, $\mathcal{C}$, will be clarified later in which case we will then consider each $\bm{\beta}_i$ to be such that $\bm{\beta}_i \sim \mathcal{N}(0,P/D)$ based on a maximum entropy argument \cite{Cover06}. As a consequence, $\|\bm{\beta}\|_2$ will concentrate around $\sqrt{P}$ with high probability which serves to model the hard power constraint with increasingly high probability as $D$ increases. Based on this set-up, it will be shown that for the classic statistical regime of $D\leq N$, the model volume is a well-defined notion (that is, the FIM is non-singular), but in the regime under investigation ($D>N$) a singular geometry arises. This implies that the volume-form can become locally degenerate, meaning that the global notion of ``the model volume'' becomes difficult to investigate. Consider now the hard power constraint, from which point we may elucidate Remark~\ref{lem:iso_vol_lem} below, wherein $\mathbb{B}^D(\sqrt{P})$ refers to the volume of a $D$-ball of radius $\sqrt{P}$. 
\begin{remark}[Regression Log Volume]\label{lem:iso_vol_lem}
\begin{align}
\log V = \frac{1}{2}\log\circ \det\left(\bm{X}^{\intercal}\bm{X}\right) + \log\left(\mathbb{B}^D(\sqrt{P})/\sigma^D\right).
\end{align}
\end{remark}

Remark \ref{lem:iso_vol_lem} is a fairly well-established idea, and the reader is encouraged to look at Gr\"unwald \cite{grunwald2007minimum}, as well as Barron, Risannen, \& Yu~\cite{Barron98} for further details. For completeness, a brief sketch for the proof of Remark \ref{lem:iso_vol_lem} is provided in Appendix \ref{app:reg:iso_vol_lem}. Remark \ref{lem:iso_vol_lem} makes clear that the $\log V$ expression can separated into two distinct components, (i) $\frac{1}{2}\log \circ\det(\bm{X}^{\intercal}\bm{X})$, which is a data dependent term, and (ii) $ \log(\mathbb{B}^D(\sqrt{P})/\sigma^D)$, which is the volume of a $D$-ball scaled against noise, which defines a degree of model \textit{distinguishability} \cite{myung2000counting}. The distinguishability term manifests here due to the volume implied by the hard power constraint. In particular, the parameter vector $\bm{\beta}$ behaves according to a ball geometry, wherein it is known that $\lim_{D\rightarrow \infty}\mathbb{B}^D(R) = 0$.\footnote{The volume of a $D$-ball, or radius $R$ is: $\mathbb{B}^D(R)=\frac{\pi^{D/2}}{\Gamma(D/2+1)}R^D$, where $\Gamma(D)\coloneqq (D-1)!$} This carries with it, the implication that the addition of new features at increasingly high $D$ may work to counteract the $\mathcal{O}(D)$ model complexity contribution, traditionally implied by AIC / BIC. However, care must be taken when extending this intuition towards large $D$, in particular for the case of $D>N$. This is because $\text{rank}\left(\bm{X}^{\intercal}\bm{X}\right)=N$, even though $\bm{X}^{\intercal}\bm{X}\in\mathbb{R}^{D\times D}$, meaning $\dim \ker (\bm{X}^{\intercal}\bm{X}) = D-N > 0$. Thus we have an increasingly large null-space as $D$ continues to increase past $N$. By implication, $\det(\bm{X}^{\intercal}\bm{X})=0$ in this range, meaning that a na\"ive volume calculation is no longer possible, and Remark \ref{lem:iso_vol_lem} looses its applicability. To circumvent this degeneracy problem, we propose the following regularization. 
\begin{definition}[Regularized Log-Volume]
\begin{align}
    \log V_{\text{reg}}\coloneqq \frac{1}{2}\log\circ \det\left(\alpha\bm{I} + \bm{X}^{\intercal}\bm{X}\right) + \log\left(\mathbb{B}^D(\sqrt{P})/\sigma^D\right), \quad \text{with } \alpha=\frac{D}{\text{SNR}}.
\end{align}
\end{definition}
There are few compelling reasons why this regularization is proposed. Firstly, it is clear that $\log V_{\text{reg}}$ is well-defined for the case of $D>N$. Moreover, this modification parallels the manner in which the identity matrix appears in standard ridge regularization solutions\footnote{$(\bm{X}^{\intercal}\bm{X} + \lambda\bm{I})\bm{\beta}=\bm{X}^{\intercal}\bm{y}.$}. As such, this form of regularization carries with it a natural geometric interpretation of extending the span of the row space of the linear map (that is, the data matrix): $\bm{X}$, from $N$ to $N+D$. Clarifying this point mathematically:
\begin{align*}
    \begin{bmatrix}\bm{X} \\ \sqrt{\alpha}\bm{I} \end{bmatrix}^{\intercal}\begin{bmatrix}\bm{X} \\ \sqrt{\alpha}\bm{I} \end{bmatrix} = \bm{X}^{\intercal}\bm{X} + \alpha \bm{I}.
\end{align*}
In other words, $\bm{X}$ originally consists of $D$ vectors, which are each $N$-dimensional. Therefore, this modification has the effect of \textit{lifting} each column vector from dimension $N$, to dimension $N+D$, with associated magnitude $\sqrt{\alpha}$, in mutually orthogonal directions. Intuitively, we are ``filling'' the degenerate dimensions with a little bit of information, to remove the degeneracy problem. Due to this, it may seem intuitively that $\log V \leq \log V_{\text{reg}}$. In fact, this is a valid conclusion, and it is clarified in the following lemma. 
\begin{lem}\label{lem:falpha_logdet}
Consider $f(\alpha) = \log\circ\det\left(\alpha\bm{I} +\bm{X}^{\intercal}\bm{X}\right)$. Then, $f(0) \leq f(\alpha)$ for $\alpha \geq 0$.
\end{lem}
\begin{proof}
For brevity, define $\bm{X}^{\intercal}\bm{X}$ as $\bm{P}$. Then proceed as follows:
\begin{align*}
    f(\alpha) &= \log\circ\det\left(\alpha\bm{I} +\bm{P}\right)
    = \log\circ\det\left(\bm{P}^{1/2}\left(\alpha\bm{P} +\bm{I}\right)\bm{P}^{1/2}\right)\\
    &= \log\circ\det \bm{P} + \log\prod_{i=1}^D (1+\alpha \lambda_i)
    =  \log\circ\det \bm{P} + \sum_{i=1}^D \log (1+\alpha \lambda_i),
\end{align*}
where $\lambda_i$ are the eigenvalues of $\bm{P}$. Since $\bm{P}$ is positive semi-definite, then each $\lambda_i \geq 0$. Thus, it is clear that $f(\alpha)$ is an increasing function (stemming from the sum of logarithms), and that $f(0)\leq f(\alpha)$, wherein equality arises whenever $\bm{P}$ itself is the zero matrix.
\end{proof}

Recall that the original motivation over MDL is to define a two-part coding scheme over some data, $\mathcal{D}$, via an appropriate code length function, $L:\mathcal{H}\rightarrow \mathbb{R}$ such that: \begin{align*}L_{\text{MDL}}(D) = \min_{H\in\mathcal{H}}\left(L(H) + L(D\mid H)\right).\end{align*} In this setting, the proposed modification of $f(0) \rightarrow f(\alpha)$ primarily affects the $L(H)$ term, and it does this by considering a regularized form of data covariance over $\bm{X}$ as: $\alpha \bm{I}_D + \bm{X}^{\intercal}\bm{X}$. However, the $L(D\mid H)$ term is left unmodified, as it still uses the original $\bm{X}$ matrix. Therefore, $L(H)$ alone will be lengthened by a factor of $\|f(\alpha) - f(0)\|_{\infty}=\|\sum_{i=1}^D \log (1+\alpha \lambda_i)\|_{\infty}$. In other words the proposed coding scheme will result in slightly longer codes as determined by the logarithm of the $\alpha$-scaled principal components of a covariance matrix over input data matrix $\bm{X}$. 
 
Lastly, a principle motivation in defining the $\log V_{\text{reg}}$ term in this way is because of its intimate link to the \textit{channel capacity}, $\mathcal{C}$. Briefly, the channel capacity is a notion which arises from communication theory. In particular, this line of research is often concerned with the same underlying system: $\bm{y} = \bm{X\beta} +\bm{\varepsilon}$, however, the interpretations of the two terms, $\bm{X}$ and $\bm{\beta}$, are somewhat different \cite{Tulino04}. In particular, $\bm{\beta}$ is interpreted as an input signal, and $\bm{X}$ is the signal mixing matrix. In other words, the \textit{known} quantity in this literature is $\bm{\beta}$, and the \textit{unknown} quantity $\bm{X}$ acts like a confounding factor, which is contrary with the traditional statistical goal of regression, whereby $\bm{X}$ is \textit{known} data, and  $\bm{\beta}$ are the \textit{unknown} coefficients. Moreover, the signal $\bm{\beta}$, is typically power constrained in expectation: $\mathbb{E}[\bm{\beta}^{\intercal}\bm{\beta}]\leq P$, which in statistics often arises due to regularization, or through consideration of a prior distribution over $\bm{\beta}$~\cite{tse2005fundamentals,bishop2006pattern}. It can be proven that $\mathcal{C}$ acts as an upper-bound on the rate of reliable information transfer for an isotropic communication channel, and that it is mathematically equivalent to the supremum of the mutual information: $\sup_{p(\bm{\beta})} \mathcal{I}(\bm{\beta};\bm{y})$, wherein $\bm{\beta}$ and $\bm{y}$ are to be considered as random variables \cite{clarke1994jeffreys}. For the problem of isotropic linear regression with real feature variables, $\mathcal{C}$ follows Theorem \ref{thm:chan_cap}. This is a well established result in the case of complex random matrices \cite{telatar1999capacity}. For the sake of completion however, we present a brief proof on its applicability for real-valued random matrices in Appendix~\ref{app:sec_chan_cap_proof}. Note that the given proof invokes classic information theory results, which contrasts the proof direction of Telatar \cite{telatar1999capacity}.
\begin{theorem}[Channel Capacity] \label{thm:chan_cap}
The channel capacity for the system $\bm{y} = \bm{X\beta} + \bm{\varepsilon}$ with $\bm{X}_{ij}\sim\mathcal{N}(0,1)$, $\bm{\varepsilon}\sim\mathcal{N}(\bm{0},\sigma^2 \bm{I})$, where $\text{SNR}=P/\sigma^2$ refers to the signal-to-noise ratio, and where $\bm{\beta}$ is a maximum entropy prior probability distribution, is given as,
\begin{equation}\label{eqn:chan_cap_thm1}
   \mathcal{C} =  \frac{1}{2}\mathbb{E}\left[\log\circ \det\left(I_{N} + \frac{\text{SNR}}{D}\bm{X} \bm{X}^{\intercal}\right)\right].
\end{equation}
\end{theorem}
If one constrains the second moment over $\bm{\beta}$, and if each $\bm{\beta}_i$ are chosen independently, then, without loss of generality, the maximum entropy distribution should be Gaussian for the proposed linear regression problem setting. In particular, for the proposed setting a natural second moment constraint is taken to be as $\mathbb{E}[\bm{\beta}^{\intercal}\bm{\beta}]\leq P$, which then implies that the natural distribution choice over $\bm{\beta}$ (to maximize mutual information around this constraint) is for each $\bm{\beta}_i\sim\mathcal{N}(0,P/D)$. 

The importance of considering $\mathcal{C}$ in the context of $\log V_{\text{reg}}$ is two-fold. Firstly, there is an established, and intimate connection between $\mathcal{C}$ and the $D/2\log (N/2\pi e)+ \log V$ term in Equation \eqref{eqn:riss}, which is asymptotic in $N$ \cite{clarke1994jeffreys}. Brief clarification of this point is provided in Appendix \ref{app:sec_baron_clark}. Secondly, $\mathcal{C}$ is quantity that is highly amenable for upper and lower-bounding (see Appendix \ref{app:sec:chan_cap_lower_upper_bound}), which will be shown to be invaluable when developing intuitions for the $D>N$ regime. Based on this, we shall consider a \textit{mean-regularized} $\log V$ expression, which will be concisely written as $\mathbb{E}[\log V_{\text{reg}}]$ (after stating Theorem \ref{thm:mean_reg_vol}, wherein its full form is shown).
\begin{theorem}[Mean-Regularized Log-Volume]\label{thm:mean_reg_vol}
\begin{align}
    \mathbb{E}_{\bm{X}_{ij}\sim\mathcal{N}(0,1)}[\log V_{\text{reg}}] = \mathcal{C} + \log\left(\mathbb{B}^D(\sqrt{P})/\sigma^D\right)  - \frac{D}{2}\log \alpha
\end{align}
\end{theorem}
\begin{proof}
Consider the following,
\begin{align*}
    \frac{1}{2}\log\circ \det\left(\alpha\bm{I}_D + \bm{X}^{\intercal}\bm{X}\right) &= \frac{1}{2}\log\circ \det\left(\alpha\left(\bm{I}_D + \alpha^{-1}\bm{X}^{\intercal}\bm{X}\right)\right) \\
    &= \frac{1}{2}\log\circ \left(\alpha^D \det\left(\bm{I}_D + \alpha^{-1}\bm{X}^{\intercal}\bm{X}\right)\right) \\
    &= \frac{1}{2}\log\circ \det\left(\bm{I}_D + \alpha^{-1}\bm{X}^{\intercal}\bm{X}\right) + \frac{D}{2}\log \alpha 
\end{align*}
Invoking the Weinstein–Aronszajn identity we write, \begin{align*}\log\circ \det\left(\bm{I}_D + \alpha^{-1}\bm{X}^{\intercal}\bm{X}\right) = \log\circ \det\left(\bm{I}_N + \alpha^{-1}\bm{X}\bm{X}^{\intercal}\right),\end{align*} with $\alpha^{-1} = \text{SNR}/D$, thereby completing the proof.
\end{proof}
\noindent In other words, if one averages over the randomness found in $\bm{X}$, then $\log V_{\text{reg}}$ will include the effects of $\mathcal{C}$. This is desirable, because $\mathcal{C}$ is a well-understood term. In fact, Xu \&
Raginsky \cite{xu2017information} have shown that model generalization error relates intimately to the notion of input-output
mutual information---the supremum of which is known to be $\mathcal{C}$. 

Before diving into the modern regime of $D>N$, let us consider how the $\mathbb{E}[\log V_{\text{reg}}]$ expression behaves in the classical regime of $D \leq N$. Ideally it should provide results which at least run parallel with classical understandings. This notion is expressed in Lemma \ref{lem:classical_understanding}.
\begin{lem}\label{lem:classical_understanding}
In the classical statistical regime such that $D \le N$,
\begin{align}
    \mathbb{E}[\log V_{\text{reg}}] \leq \frac{D}{2}\log N + \log\left(\mathbb{B}^D(\sqrt{P})/\sigma^D\right).
\end{align}
\end{lem}
\begin{proof}
It is clear that the $f(\alpha)$ expression shown in Lemma \ref{lem:falpha_logdet} is a concave function, due to the summation of logarithms. Therefore $\mathcal{C}$ behaves as a concave function, meaning it is possible to invoke Jensen's inequality. As such one can use the upper-bound developed in Appendix \ref{app:sec:chan_cap_lower_upper_bound}, so that:
\begin{align*}
\mathbb{E}[\log V_{\text{reg}}] &\leq \frac{D}{2}\log(\alpha^{-1} N + 1) +\frac{D}{2}\log \alpha + \log\left(\mathbb{B}^D(\sqrt{P})/\sigma^D\right) \\
&= \frac{D}{2}\log (N + \alpha) \log\left(\mathbb{B}^D(\sqrt{P})/\sigma^D\right)\\
&\stackrel{(i)}\approx \frac{D}{2}\log N + \log\left(\mathbb{B}^D(\sqrt{P})/\sigma^D\right),
\end{align*}
where (i) assumes that $N\gg \alpha$. 
\end{proof}
In Lemma \ref{lem:classical_understanding}, the assumption of $N\gg D/\text{SNR}$ posits that there is a sufficiently large quantity of data, relative to $D$ (notice also that $\text{SNR}$ is a term strongly driven by $D$ since it is a function $\|\bm{\beta}\|_2^2$). Thus, we see that the $\mathbb{E}[\log V_{\text{reg}}]$ expression has two driving terms, the first of which behaves precisely as the classic $\mathcal{O}(D)$ term in AIC / BIC statistics, and the latter expression which encodes the effect of the hard power constraint---which can interpreted as a \textit{strict} regularization over the parameter space. Therefore the application of $\mathbb{E}[\log V_{\text{reg}}]$ in the classical regime provides an encouraging picture when looking to extend it for $D > N$. In this respect, we develop Lemma \ref{lem:modern_regime}, which works similarly by considering an upper-bound over $\mathcal{C}$ (based on the bounds of Appendix \ref{app:sec:chan_cap_lower_upper_bound}). 
\begin{lem}\label{lem:modern_regime}
In the modern statistical regime such that $D > N$,
\begin{align}
    \mathbb{E}[\log V_{\text{reg}}] \leq \frac{N}{2}\log (\text{SNR} + 1) + \log\left(\mathbb{B}^D(\sqrt{P})/\sigma^D\right) + \frac{D}{2}\log \alpha. \label{eqn:modern_regime}
\end{align}
\end{lem}
\begin{proof}
The proof follows as a direct application of the upper-bound developed in Appendix \ref{app:sec:chan_cap_lower_upper_bound}.
\end{proof}
Of particular interest is the observation that the upper-bound on $\mathcal{C}$ (first term of Equation \ref{eqn:modern_regime}) is no longer a function of $D$. In other words, the channel capacity \textit{saturates} with respect to $D$ when modeling over $D>N$. Although this is a classic result in information theory literature \cite{telatar1999capacity}, its inclusion in Lemma \ref{lem:modern_regime} to help ``circumvent'' the pathological singularities which manifest in the modern regime is novel. In fact, it will be shown shortly that this saturation helps establishing a nice geometric intuition in the $D>N$ regime. However, we first consider Theorem \ref{thm:final_full}, which includes the impact of the $\mathcal{O}(D)$ term on $\mathbb{E}[\log V_{\text{reg}}]$ (recall from Equation~\eqref{eqn:riss} that for a complete understanding on ``generalization behavior'' it is necessary to consider $\log V$ in tandem with $\mathcal{O}(D)$).
\begin{theorem} \label{thm:final_full}
In the modern statistical regime such that $D > N$,
\begin{align}\label{eq:ubound}
    \frac{D}{2}\log \frac{N}{2\pi e} + \mathbb{E}[\log V_{\text{reg}}] \leq \frac{D}{2} \log \frac{ND}{2\pi e} + \frac{N}{2}\log(\text{SNR} + 1) + \log \mathcal{B}^D (1).
\end{align}
\end{theorem}
\begin{proof}
\begin{align*}
 \frac{D}{2}\log \frac{N}{2\pi e} +  \mathbb{E}\left[\log V_{\text{reg}}\right] &\leq \frac{D}{2}\log \frac{N}{2\pi e} + \frac{N}{2}\log(\text{SNR} + 1) + \log \left(\mathcal{B}^D(\sqrt{P})/\sigma^D\right) + \frac{D}{2}\log \alpha \\
 &= \frac{D}{2}\log \frac{N}{2\pi e} + \frac{N}{2}\log(\text{SNR} + 1) + \log \left(\mathcal{B}^D(1)\cdot \sqrt{\text{SNR}^D}\right) + \frac{D}{2}\log \alpha \\
 &= \frac{D}{2}\log\frac{N\alpha}{2\pi e } + \frac{N}{2}\log(\text{SNR} + 1) + \frac{D}{2}\log \text{SNR} + \log\mathcal{B}^D(1) \\
 &= \frac{D}{2}\log\frac{ND}{2\pi e} +  \frac{N}{2}\log(\text{SNR} + 1) + \log\mathcal{B}^D(1).
\end{align*}
\end{proof}
Theorem \ref{thm:final_full} provides an indication on the different factors which ultimately contribute to the encoding (compressibility) of the isotropic, power-constrained regression hypothesis space under a modern, statistical viewpoint. The lower the value of this expression, the better the generalization properties of the chosen model, based on an Occam's Razor principle (smaller, simpler models should be able to generalize better). Interestingly, each term established in Theorem \ref{thm:final_full} (right-hand side of Equation~\eqref{eq:ubound}) proposes a distinct viewing angle to this problem. In particular:

\begin{itemize}
    \item The first term $\frac{D}{2}\log\frac{ND}{2\pi e}$: This is simply the standard AIC / BIC complexity term, except now it is weighted against an additional $\log D$ factor. That is, in the modern regime (for the proposed model), this term increases in the order of $\mathcal{O}(D\log D)$ instead of $\mathcal{O}(D)$.
    \item The second term $\frac{N}{2}\log(\text{SNR} + 1)$: This term refers to the fact that the channel capacity saturates. As will be clarified shortly, this term carries with it a strong geometric interpretation, which is driven by the notion of sphere packing phenomenon. If one considers $\bm{X}$ as a linear map $\bm{X}: \mathbb{R}^D \rightarrow \mathbb{R}^N$, then this term relates how the $D$-dimensional ball parameters help in \textit{filling} the lower $N$-dimensional data space with respect to the linear map $\bm{X}$.  
    \item The third term $\log\mathcal{B}^D(1)$: This term clarifies the impact of the proposed power constraint on the parameters, which can be considered to as strong form of ridge regularization on $\bm{\beta}$. In particular, this term would seem to suggest that there is a \textit{smoothing effect} over the parameter space in the $D>N$ regime. In other words, the power constraint (consequently, sufficiently strong ridge regularization) forces ``mass'' to be evenly distributed among the $D$ coefficients. If one considers the maximum entropy interpretation of $\bm{\beta}$, in particular each $\bm{\beta}_i\sim\mathcal{N}(0,P/D)$, evidently as $D$ increases, each sampled $\bm{\beta}_i$ will begin to concentrate closer around zero. Numerically, this manifests through the notion that  $\lim_{D\rightarrow \infty}\log\mathcal{B}^D(1)\rightarrow -\infty$, which serves to act as a counter-balance for the traditional $\mathcal{O}(D)$ complexity terms in AIC and BIC. 
\end{itemize}
%
\begin{figure}[t]
\centering
\begin{subfigure}{0.45\textwidth}
\centering
\includegraphics[width=\linewidth]{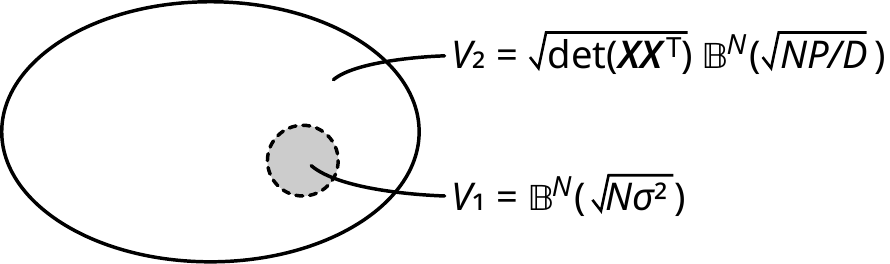} 
\caption{Volume ratios for ellipsoid sphere packing in linear regression.}
\label{fig:subim1}
\end{subfigure}\hspace*{10pt}
\begin{subfigure}{0.45\textwidth}
\centering
\includegraphics[width=\linewidth]{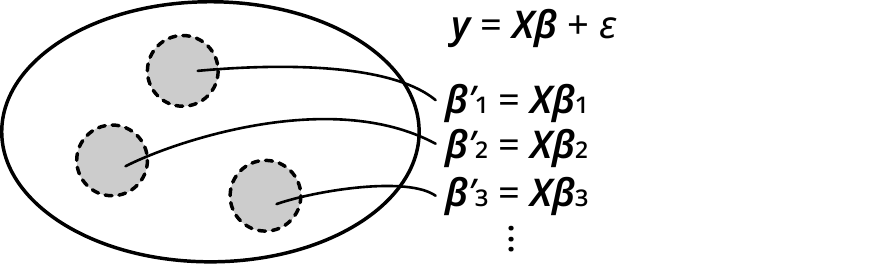}
\caption{New $\bm{\beta}$ inputs will map to different locations in the encasing $\bm{y}$ ellipsoid. These are $N$-balls of approximate radius $\sqrt{\mathbb{E}[\bm{\varepsilon}^{\intercal}\bm{\varepsilon}]}$.}
\label{fig:subim2}
\end{subfigure}
\caption{$N$-dimensional ellipsoid sphere packing. \label{fig:ellipse_pack}}
\label{fig:image2}
\end{figure}
As was mentioned in the above points, the $ \frac{N}{2}\log(\text{SNR} + 1)$ term is one which is intrinsically driven by geometric sphere packing argument. A common example of sphere packing as it arises in additive white Gaussian noise (AWGN)  \cite{tse2005fundamentals} is provided for the reader in Appendix \ref{app:sec:sphere_pack_AWGN}. Here however, we focus on showing how it manifests in $\mathcal{C}$, considering the model $\bm{y} = \bm{X\beta} + \bm{\varepsilon}$. Initially, we shall consider $\bm{X}$ to be fixed (that is, non-random). We base this calculation on two parts: (i) Considering what the model is trying to achieve without noise: $\bm{y}=\bm{X\beta}$, and (ii) The impact of noise. Considering (i), we will consider each $\bm{\beta}_i \sim \mathcal{N}(0,P/D)$, which implies that $\bm{y}=\bm{X}\bm{\beta}$ will have an ellipsoidal shape. That is, $\bm{\beta}$  will be scaled and rotated due to $\bm{X}$. Moreover, since $\bm{X}$ is a linear map, it is known that: $\text{Vol}(\bm{X \beta}) = \sqrt{\det\left(\bm{X X}^{\intercal}\right)}\text{Vol}(\bm{\beta})$. Finally, observe that since $\bm{X}: \mathbb{R}^D\rightarrow \mathbb{R}^N$, and that each $\bm{\beta}_i\perp\bm{\beta}_j$, we are ultimately dealing with a space wherein $\text{Vol}(\bm{\beta})=\mathbb{B}^N(\sqrt{NP/D})$. This results in  the value of $\text{Vol}(\bm{X \beta})\coloneqq V_2 = \sqrt{\det(\bm{X}\bm{X}^{\intercal})}\cdot\mathbb{B}^N(\sqrt{NP/D})$, in which most vectors transformed under $\bm{X\beta}$ will lie inside with high probability. Considering now point (ii), since $\text{dim}(\bm{\varepsilon})=N$, and $\bm{\varepsilon}\sim\mathcal{N}(\bm{0},\sigma^2I_N)$, the volume induced by the $N$-ball in regards to the presence of noise is, $V_1 = \mathbb{B}^N(\sqrt{N\sigma^2)}$. Ultimately a log-volume ratio can be calculated as: $ \log(V_2/V_1)= (1/2)\log \circ \det (\text{SNR}\cdot \bm{X}\bm{X}^{\intercal}/D)$. Now, if one considers each $\bm{X}_{ij}$ as being sampled iid from an underlying probability distribution, then with each newly sampled $\bm{X}$ a slightly different sphere packing problem is performed, as $\bm{X}$ manifests via the enlargement factor: $\sqrt{\det{\bm{XX}^{\intercal}}}$. Therefore, if one takes an expectation over $\sqrt{\det{\bm{XX}^{\intercal}}}$, then one is  effectively considering the \textit{average enlargement factor} over $\text{vol}(\bm{\beta})$. Thus, it is clear that one derives an upper-bound on $\mathcal{C}$ for the case of $\text{SNR} \gg D$. Note, also that a non-random $\bm{X}$ is indeed a valid (and classic) interpretation for $\mathcal{C}$ (see section 3.2 in~\cite{telatar1999capacity}).

Due to the sphere packing saturation in $\mathcal{C}$, it is clear that the main factors affecting the MDL code-length in the $D>N$ regime are that of (i) $\frac{D}{2}\log \frac{N}{2\pi e}$, and (ii) $\log\mathcal{B}^D(1)$. However, notice that $\mathcal{B}^D(1) \rightarrow 0$ as $D\rightarrow \infty$ at a factorial rate, and thus it would appear that generalization error should not necessarily \textit{explode} if one keeps increasing $D$---as per the usual classical conclusions implied by AIC and BIC. As a consequence, if generalization error does not increase catastrophically, one would anticipate the test error rates to either stabilize around some region (possibly asymptotically), or to steadily decrease towards zero. On this point, when working with regularized linear models the presence of an asymptotic behavior has been recently found and studied in the work of Hastie et al.~\cite{hastie2019surprises}. Moreover, in this work it is noted this asymptotic behaviour is a function of SNR, which is a term that appears very naturally in the proposed geometric formulation of this problem. In fact from Theorem \ref{thm:final_full}, a few conclusions can be drawn. These are: (i) As SNR decreases, the generalization capability should increase, (ii) As $D$ increases far beyond $N$, if regularization is strong (thereby enforcing the proposed power constraint), then one does not necessarily expect poor generalization ability (that is, we do not expect test error to continue to explode). The condition of strong regularization ensures that that the ball geometry term, $\log \mathcal{B}^D(1)$ is applicable. Finally, (iii) Increasing $N$ in the $D>N$ regime does not appear as a strong influence. In particular, $D$ increases its effect is logarithmic ($\log N$), and if $N$ and $\text{SNR}$ are fixed, then the saturating effects of $\mathcal{C}$ are fixed. Ultimately, it seems to mainly add an additional overhead to the designated code-length of the models in the hypothesis space. This last point is interesting to make note of, since there is a commonly held belief that ``more data is \textit{always} better''. 

Finally, although it would be tempting to apply these results directly into learning theory bounds on the generalization error, the reader is reminded that MDL theory is by principle designed to primarily address the problem of data compression (with the intuition that simpler models \textit{should} generalize better based on an Occam's Razor principle). Therefore the MDL code-length is not necessarily amenable for use in such bounds.\footnote{MDL code-length can be used in the case of assuming zero-training error, but this is not a situation which arises naturally here, as training error is clearly non-zero in $D>N$ \cite{seldin2009pac}.}

Empirical results for the out-of-sample test error in the case of isotropic linear regression model, and the impact of the magnitude of $\|\bm{\beta}\|_2^2$ is made clear in Figure~\ref{fig:double_desc_reg_curves}. To generate these results we used  {\footnotesize\textsf{sklearn.linear\_model.Ridge}}, as (i) The ridge regression hyper-parameter, $\alpha$, is used to emulate the effect of a power constraint, and (ii) We intend to keep the code as simple as possible to maximize reproducibility. For these experiments, we generated data according to $\bm{X}_{ij}\sim\mathcal{N}(0,1)$, and we assumed that there exists a \textit{true} underlying generating process, of some dimension, which we choose wlog as 150. We keep increasing $D$ (the number of feature variables of the proposed model fit), up to and beyond 150, ranging from $D=1$, up to $D=2500$. In this way, we transition from \textit{tall} $\bm{X}$ matrices to \textit{long} $\bm{X}$ matrices, and are able to demonstrate a very clear transition from the classical, to the modern train-test risk regimes. In addition, 10-fold cross validation is performed in order to produce the train-test curves, and we opt for three values of the ridge hyper-parameter: $\alpha = (10^{-2},10^0,10^2)$, chosen wlog, in order to demonstrate how the double descent behaviour can be slowly switched on-and-off, and considered $N=(300,600,900)$ datapoints. Moreover, the random seed is fixed between runs, and each $\bm{\beta}_i$ is sampled from $\mathcal{N}(0,P=0.25)$, with additive noise of $\bm{\varepsilon}\sim\mathcal{N}(0,1)$, so that we work with a fixed $\text{SNR}$ between each case. Lastly, the experiments were run on a 64 bit Windows OS, with Intel(R) Core(TM) i7-7500U CPU @ 2.70 GHz, since these experiments are designed to be straight-forward, and reproducible by everyone.
\begin{figure}[t]
  \centering
  \includegraphics[ width=\linewidth]{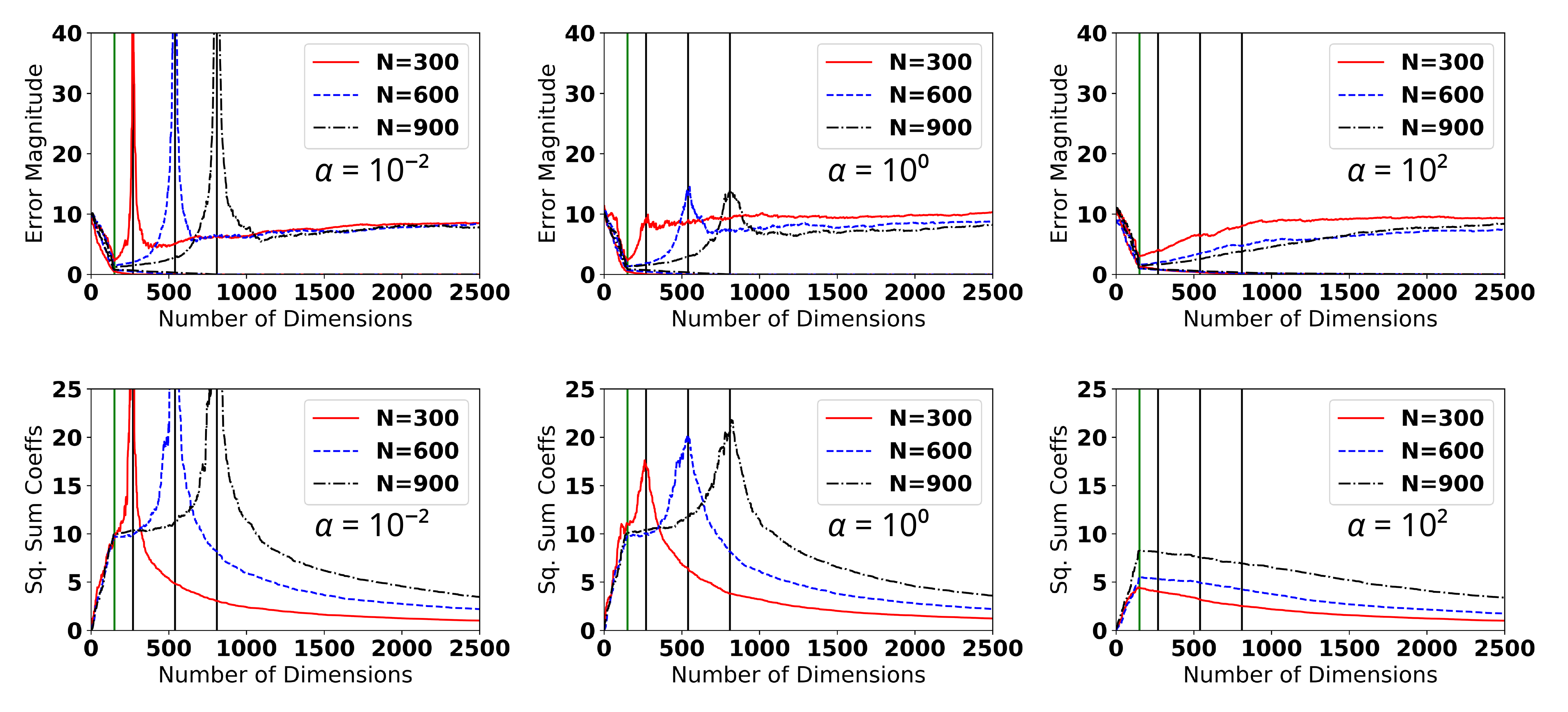}
  \caption{How double descent risk manifests for $N=(300,600,900)$ and $\alpha = (10^{-2},10^0,10^2)$. Risk curves (upper row) and $\|\bm{\beta}\|_2^2$ plots (bottom row).}
  \label{fig:double_desc_reg_curves}
\end{figure}

Firstly, in Figure \ref{fig:double_desc_reg_curves} we can see that the classical (U-shaped), and modern (double descent) regimes are visible. As noted by many researchers, the peak of the observed double descent phenomenon (the \textit{interpolation threshold}), occurs at $N=D$~\cite{belkin2019reconciling,belkin2019reconciling,hastie2019surprises}. When considering the train-test curves based on, $\alpha=10^{-2}$, it is evident that out-of-sample test error indeed peaks at $N=D$. However, for continually increasing $D$, all the test risk curves reach a point of saturation, regardless of the choice of $\alpha$. This was implied to occur due to the saturating effect of $\mathcal{C}$, when $D\rightarrow\infty$. Moreover, it is clear that for each run, the train-test curves asymptote to the same error magnitude, which results due to the fixed $\text{SNR}$ between each run. Such a behaviour is predicted to occur from Hastie et al. \cite{hastie2009elements}, wherein they claim that the asymptotic behaviour for risk curves (in linear problems), is highly dependent upon $\text{SNR}$. And indeed here, $\text{SNR}$ arises naturally in $\log V$, and when considering $\mathcal{C}$. In addition to this, the peaking behaviour of double descent risk coincides with the peaking locations of $\|\bm{\beta}\|_2^2$, and it scales strongly dependent upon the choice of $\alpha$. Such a behaviour was anticipated to occur via the considerations of model distinguishability, and sphere packing. In other words, the outer ellipsoid: $V_2 = \sqrt{\det(\bm{X}\bm{X}^{\intercal})}\cdot\mathbb{B}^N(\sqrt{NP/D})$, has a volume dependent on $P$, which is in turn reflective of the maximum magnitude allowed by $\|\bm{\beta}\|_2^2$. Evidently if $P$ \textit{explodes}, then, $V_2$ will explode in volume, relative to the noise spheres: $V_1$, and the relative sizes of these volumes is indicative of out-of-sample test error. In other words, when only weak regularisation is applied ($\alpha=10^{-2}$), it is not possible to properly constrain the magnitude of $\|\bm{\beta}\|_2^2$, especially at the interpolation threshold ($N=D$), and thus $V_2$ like-wise explodes in magnitude. However as we progressively strengthen the effect of regularisation: $\alpha = (10^0,10^2)$, we observe a gradual reduction, and even a complete elimination of the double descent risk pathology. Recently, Nakkiran et al. \cite{nakkiran2020optimal} have also investigated the impact of ridge regularization on the double descent peaking phenomenon, wherein they observe that an \textit{optimally tuned} regularizer can work to eliminate the presence of this peaking behaviour. However, the intuition behind this has not yet been elucidated from a geometric angle. 

In addition, the double descent peak is observed to shift to the right with increasing $N$. Information theoretically, increasing $N$ proliferates the total number of possible encodings which may be able to explain the observed data. This is an interpretation which is consistent with Rissanen's original derivation of MDL, in which he states: ``the number of distinguishable models grows with the length of the data, which seems reasonable. In view of this we define
the model complexity (as seen through the data)'' ~\cite{rissanen1996fisher}. Interestingly, this can imply that when a model's total distinguishability is insufficient (weak regularization, and or insufficient noise), it is possible for the model to generalize well on one quantity of data, $N_1$, and then upon re-training on some new data such that $N_2 > N_1$, to then generalize poorly, due to the ability of the double-descent cusp to shift towards the right. Similar ideas have been uttered recently by Nakkiran et al., in that: ``for a fixed architecture and training procedure, more data (can) actually hurt''~\cite{nakkiran2019deep}. 

Finally, we make clear that the concept of distinguishability (as it stems from $\log V$), has an intuition which is shared by many schools of thought on double descent risk, and generalization error. For example: (i) It bears strong intuitive similarities to the idea of a \textit{jamming transition}, which refers to the tight packing of particles in physics, when a material transitions from fluid to solid. This idea is used by Geiger et al.~\cite{geiger2019jamming,geiger2020scaling} to clarify several ideas on double descent risk. Moreover, (ii) Many statistical generalization theorems are based on the notion of ``sphere coverings'' over constrained function spaces, in which the sphere covering number, and sphere packing number, closely relate~\cite{vapnik2013nature}. 


\section{Statistical Lattice Models and Model Volumes}

Statistical lattice models are popular, traditional machine learning models which include \emph{Boltzmann machines} (or Ising models)~\cite{Ackley85}, log-linear models~\cite{Amari01}, and the matrix balancing problem~\cite{sugiyama2017tensor}. In this section, we will work the hierarchical encoding of a probability distribution via a \emph{lattice} structure \cite{sugiyama2016information}, which will be shown to naturally lead into the $\eta$ co-ordinates ($m$-flat) from information geometry (see \S\ref{sec:info_geo}), and analyze the learning of distributions over a lattice structured domain.

Formally, a \emph{partially ordered set} (\emph{poset}) is a tuple, $(\mathcal{P},\leq_{\mathcal{P}})$, where $\mathcal{P}$ is a set of elements, and $\leq_{\mathcal{P}}$ denotes an ordering structure, such that (1) $\forall p\footnote{Since elements from a poset are denoted by $p$, we denote a probability measure as $\mathbb{P}$ when referencing poset systems.} \in \mathcal{P},$  $p \leq_{\mathcal{P}} p$ (reflexivity), (2) If $p\leq_{\mathcal{P}} q$, and $q\leq_{\mathcal{P}} p$, then $p=q$ (antisymmetry), and (3) If $p\leq_{\mathcal{P}} q$, and $q \leq_{\mathcal{P}} r$, then $p \leq_{\mathcal{P}} r$. Note that not every element may be directly comparable to every other element in the set (which would be known as a \textit{total} ordering). In addition, a poset $(\mathcal{P}, \leq_{\mathcal{P}})$ is called a \emph{lattice} if every pair of elements $p, q \in \mathcal{P}$ has the least upper bound $p \vee q$ and the greatest lower bound $p \wedge q$~\cite{Davey02}.
We assume that $\mathcal{P}$ is finite. In working with posets it is common to consider the zeta function, $\zeta: \mathcal{P}\times \mathcal{P} \rightarrow \{0,1\}$ such that $\zeta(p,q) = \mathbf{1}_{p\leq q}$~\cite{Gierz03}. The lattice structure always gives us the $\theta$ and $\eta$ co-ordinates of a statistical manifold: $\log \mathbb{P}(p) = \sum_{q \in \mathcal{P}} \zeta(q, p) \theta_q = \sum_{q \le p} \theta_q$ and $\eta_p = \sum_{q \in \mathcal{P}} = \zeta(p, q) \mathbb{P}(q) = \sum_{q \ge p} \mathbb{P}(q)$~\cite{sugiyama2017tensor}.
For example, for Boltzmann machines with $d$ binary variables, the lattice space $\mathcal{P} = \{0, 1\}^n$, where $p = (p_1, \dots, p_n) \leq_{\mathcal{P}} q = (q_1, \dots, q_n)$ if $p_i \le q_i$ for all $i \in \{1, \dots, n\}$.
The size $D = |\mathcal{P}| = 2^n$ in this case. The metric tensor for the information manifold of the proposed lattice structure is shown in Theorem \ref{thm:sugiyama_metric}, which was previously derived by Sugiyama et al. \cite{sugiyama2017tensor}.
We assume that $\mathcal{P} = \{1, \dots, |\mathcal{P}|\}$ such that $1$ corresponds to the least element without loss of generality.
\begin{theorem}[Lattice Metric Tensor \cite{sugiyama2017tensor}] \label{thm:sugiyama_metric}$    \mathcal{G}_{ij} = \sum_{p\in\mathcal{P}}\zeta(i,p)\zeta(j,p)\mathbb{P}(p) - \eta_i\eta_j = \eta_{i \vee j} - \eta_i\eta_j$.
\end{theorem}\vspace*{-5pt}
 In Theorem~\ref{thm:sugiyama_metric}, we can replace $\sum_{p\in\mathcal{P}}\zeta(i,p)\zeta(j,p)\mathbb{P}(p)$ with $\eta_{i \vee j}$ as we assume that $\mathcal{P}$ is a lattice and $\eta_{i \vee j}$ always exists, which says that we only consider those poset structures in which the $\eta$ co-ordinate values are shared (nested) between $\eta_i$ and $\eta_j$ for the off-digagonal terms in the metric tensor.
 In Theorem \ref{thm:sugiyama_metric} it is clear that this metric tensor is expressed in terms of the $\eta$ co-ordinates. The equivalent metric tensor in terms of the $\theta$ co-ordinates is available, but it is much more difficult to work with (requires the M\"obius function instead of the zeta function). Moreover, in this definition of the metric tensor, the first row and column are always zeros, resulting in again, a singular geometry. Luckily this time however, since $-\theta_1$ corresponds to the partition function, it is generally removed in practice,~\cite{sugiyama2018legendre} so that we effectively work with co-ordinates $\bm{\eta}' = (\eta_2, \dots, \eta_D)$, resulting in $\mathcal{G}'\succ 0$.
 
 Based on this set-up, it is possible to derive the upper and lower bounds for $\log V$. For lattice models, the $\eta$ co-ordinates lie compactly on a simplex within the unit $D$-hypercube, that is $\Omega = [0,1]^D$, which makes evaluations much simpler~\cite{sugiyama2016information}. In fact, it is possible to perform the re-parameterization: $\bm{\delta}=f(\bm{\eta})$, which allows $\bm{\delta}$ to be sampled from a Dirichlet distribution. This makes the $\eta$ co-ordinate more intuitive to work with, and provides us with a tractable way to evaluate the volume integral via sampling. We provide details of this re-parameterization in Appendix \ref{app:sec:reparam_poset}. The $\log V$ bounds which result are shown in Theorem \ref{thm:poset_vol_bound}, with the proof clarified in Appendix \ref{app:sec:poset_bounds_proof}.
\begin{theorem}[Lattice Log Volume Bounds] \label{thm:poset_vol_bound} $\log V$ is bound as in Equation~\eqref{eqn:upper_lower_poset}, where $\mathcal{G}=\mathcal{M}^{\intercal}\mathcal{M}$, $\bm{\delta}=f(\bm{\eta})$ is a re-paramterization, and $\Gamma(D)=(D-1)!$ is the standard Gamma function.
\begin{equation}\label{eqn:upper_lower_poset}
\footnotesize{  \overbrace{\mathbb{E}\left[\sum_{i=1}^D \log\left(\mathcal{M}_{ii}(\bm{\delta})\right)\right]}^{\text{\tiny{``Richness''}}} + \overbrace{\log\left(\frac{1}{\Gamma(D)}\right)}^{\text{\tiny{``Distinguishability ''}}} \leq \log V \leq \overbrace{\log\left(\mathbb{E}\left[ \sqrt{\prod_{i= 1}^D \mathcal{G}_{ii}(\bm{\delta})}\right]\right)}^{\text{\tiny{``Richness''}}} + \overbrace{\log\left(\frac{1}{\Gamma(D)}\right) }^{\text{\tiny{``Distinguishability ''}}} }.
\end{equation}
\end{theorem}
\begin{proof}
Proof is given in Appendix~\ref{app:sec:poset_bounds_proof}.
\end{proof}
As Theorem \ref{thm:poset_vol_bound} makes clear, the $\log V$ term can be decomposed into two components which we define as: (i) \textit{Model richness}: which is driven by the elements found in the metric tensor, and (ii) \textit{Model distinguishability}: which refers to the volume of a \textit{probability simplex}. Intuitively, (i) For (higher-order) Boltzmann machines, multi-way interactions can be encoded in a desired lattice structure, and this in turn will drive the construction of the metric tensor via the $\eta$ co-ordinate system (Theorem \ref{thm:sugiyama_metric}). Thus, since the metric tensor depends strongly upon the chosen lattice, and since it is used in defining angles and geodesics over a manifold, it would appear that this expression encodes a notion of \emph{model richness} and or expressibility over the manifold. As for point (ii), since the $\eta$ co-ordinate system is constrained to lie on a simplex geometry (as made explicit by the poset structuring) the volume is: $1 / \Gamma(D) = 1 / (D-1)!$. Evidently, as dimensionality increases the simplex volume decreases. A nice combinatorial intuition of this is that for $D$ randomly sampled numbers $\{n_i\}_i^D$, the probability of obtaining a permutation which is precisely the total ordering of these $D$ points (i.e. $n_1<n_2<...<n_D$) is $1/D!$. Thus, even if the AIC $\mathcal{O}(D)$ model complexity term grows without bound, the simplex constraint over the $D$ parameters serves to act as a strong counter-balance. In fact, this counter-balancing imbues the MDL expression with a \textit{double descent}-like behaviour. As $D$ increases, the value of the MDL expression first increases, and then decreases. Since MDL relates strongly to the notion of model generalizbility the double descent phenomenon which arises here paints an intriguing picture for the double descent risk phenomenon that has empirically arisen in the deep learning field. A visual example of this on toy values is made clear in Figure \ref{fig:sub-DD_logV}. 
\begin{wrapfigure}{r}[0pt]{.45\textwidth}
  \centering
  \includegraphics[width=.8\linewidth]{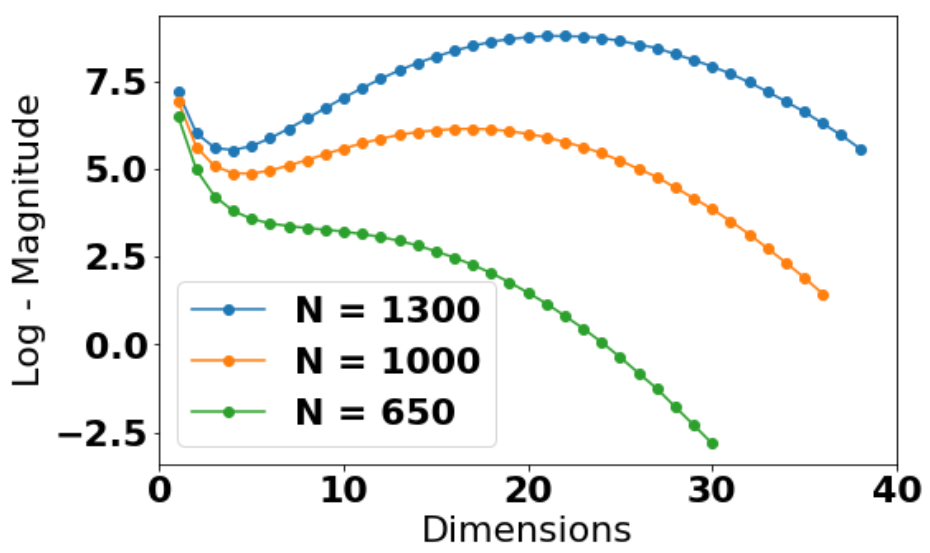}  
    \caption{Three different data cases on MDL evaluations.}
  \label{fig:sub-DD_logV}
\end{wrapfigure}
It is important to note that our analysis assumes access to a \textit{fully observable} lattice model (and thus is not immediately applicable to restricted Boltzmann machines without careful consideration). This is because in latent hierarchical settings the underlying geometry can become geometrically singular, which complicates the global model volume calculation with respect to the FIM~\cite{watanabe2009algebraic,sun2019lightlike}.

In regards to (ii) model distinguishability, it should be noted that the term \textit{distinguishability} is somewhat overloaded here, since in classic MDL literature it has been traditionally reserved for the entire expression: $\int_{\bm{\Theta}} \sqrt{\text{det}(\mathcal{I}(\bm{\theta}))}d\bm{\theta}$ \cite{balasubramanian2005mdl,rissanen1997stochastic}.
Our proposal is to elaborate it slightly by splitting this term into two terms, where one term works to clarify the importance of the model architecture, and the other term clarifies the constraints that exist in the underlying parameter space. Inspecting these terms in the limit of the over-parameterized regime results in Remark \ref{rem:volume_poset_limit_0} (see Appendix~\ref{app:sec:poset_vol_0} for its proof).
\begin{remark}[Limiting Lattice Volume]\label{rem:volume_poset_limit_0}$
    \lim_{D\rightarrow \infty } V = 0$.
\end{remark}\vspace*{-5pt}
Thus, the volume of the proposed statistical lattice models tends towards zero for large $D$, and this limit can be said to converge factorially due to the simplex volume.


\section{Conclusion}
Motivated through the works of Rissanen, and Balasubramanian, investigation of the $\log V$ term has shown that a geometrical perspective is invaluable when studying the behavior of certain models in the modern statistical regime. In particular, the $\log V$ term readily calls for a coding theory perspective when analyzing such models, which then makes clear of several saturation effects which can arise. It appears that certain models are characterized by a high degree of compressibility, from which it is possible to invoke an Occam's Razor-like principle in regards to explaining why out-of-sample test error is not necessarily catastrophic when dimensionality far exceeds the number of datapoints. 




\appendix
\section*{Appendix}

\setcounter{remark}{2}
\setcounter{theorem}{5}
\setcounter{equation}{5}
\setcounter{figure}{4}

\section{Isotropic Linear Regression}

\subsection{Fisher Information Matrix in Linear Regression}\label{app:reg:iso_vol_lem}

In order to prove Theorem \ref{lem:iso_vol_lem} we first make note of Remark \ref{rem:iso_FIM_lem}, which develops the FIM (metric tensor) for linear regression. 

\begin{remark}[Linear Regression FIM]\label{rem:iso_FIM_lem}
The FIM (Riemannian metric tensor) for the proposed linear regression model is $
     \mathcal{G}(\bm{\beta}) = \mathbb{E}\left[\bm{X}^{\intercal}\bm{X}\right]/\sigma^2.$
\end{remark}

\begin{proof}
The proposed regression can be represented probabilistically as :
\begin{alignat*}{2}
     &\quad& p(y_i|\bm{x}_i,\bm{\beta};\sigma^2) &= \frac{1}{\sqrt{2\pi}\sigma^2}\exp\left(-\frac{1}{2\sigma^2}(y_i - \bm{x}_i^{\intercal}\bm{\beta})^2\right) \\
 \Rightarrow&      &  \log p(y_i|\bm{x}_i,\bm{\beta};\sigma^2)  &= -\log(\sqrt{2\pi}\sigma^2) -\frac{1}{2\sigma^2}(y_i - \bm{x}_i^{\intercal}\bm{\beta})^2 \\
 \Rightarrow&      &  \nabla_{\beta}^2 \log p(y_i|\bm{x}_i,\bm{\beta};\sigma^2)  &= -\frac{\bm{x}_i\bm{x}_i^{\intercal}}{\sigma^2},
 \end{alignat*}
  where we have assumed variance $\sigma^2$ is known. Generalizing this to the data matrix $\bm{X}$, negating the above expression, and taking the expectation, we obtain the Fisher information matrix (Riemannian metric tensor) as required.
 \end{proof}
 
 The proof of Theorem \ref{lem:iso_vol_lem} thus completes as follows: 
 
 \begin{proof}
 Since $V=\int_{\mathcal{B}} \sqrt{\det\left(\mathcal{G}(\bm{\beta})\right)}d\bm{\beta}$, where, $\mathcal{G}(\bm{\beta})=\mathbb{E}\left[\bm{X}^{\intercal}\bm{X}\right]/\sigma^2$.
  \begin{align*}
      V &= \int_{\mathcal{B}} \sqrt{\det\left(\frac{\bm{X}^{\intercal}\bm{X}}{\sigma^2}\right)}d\bm{\beta}\\ 
      &= \sqrt{\det\left(\frac{\bm{X}^{\intercal}\bm{X}}{\sigma^2}\right)} \int_{\mathcal{B}} d\bm{\beta}\\
     &= \frac{\sqrt{\det\left(\bm{X}^{\intercal}\bm{X}\right)}}{\sigma^D}\idotsint_{\beta_1^2 + ... + \beta_D^2 \leq P} d\bm{\beta} \\
     &= \sqrt{\det\left(\bm{X}^{\intercal}\bm{X}\right)}\frac{\mathbb{B}^D(\sqrt{P})}{\sigma^D},
 \end{align*}
 Thus, arriving at $\log V = \frac{1}{2}\log\circ \det\left(\bm{X}^{\intercal}\bm{X}\right) +\log \frac{\mathbb{B}^D(\sqrt{P})}{\sigma^D}$, where $\mathbb{B}^D(\sqrt{P})$ is the volume of a $D$-ball of radius $\sqrt{P}$.
 \end{proof}

  We make clear that the condition used to evaluate the integral in Theorem \ref{lem:iso_vol_lem} is based on integrating over the domain: $\bm{\beta}^{\intercal}\bm{\beta}\leq P$, where $\bm{\beta}\in\mathbb{R}^D$. This is proposed as the hard power constraint. Due to the need to define the channel capacity, $\mathcal{C}$, later in this paper, $\bm{\beta}$ will eventually be considered to act as a random vector. This extension is necessary later when defining channel capacity, wherein it will be assumed that each $\bm{\beta}_i \sim \mathcal{N}(0,P/D)$, implying that $\bm{\beta}^{\intercal}\bm{\beta}\sim \frac{P}{D}\chi^2(D)\Rightarrow\mathbb{E}[\bm{\beta}^{\intercal}\bm{\beta}]=P$. The choice of distribution on each $\bm{\beta}_i$ is driven by a maximum entropy argument, based on the first and second statistical moments. Specifically, it is known that the maximal entropy distribution with a specified mean (which we take as zero), and a specified covariance, $\mathbb{E}[\bm{\beta}\bm{\beta}^{\intercal}]$, is Gaussian, and further this covariance naturally relates to the required power constraint through: $\mathbb{E}[\bm{\beta}^{\intercal}\bm{\beta}]=\text{tr}\left(\mathbb{E}[\bm{\beta}\bm{\beta}^{\intercal}]\right)$. Thus, based on a maximum entropy argument we opt for $\bm{\beta}_i \sim \mathcal{N}(0,P/D)$. Additionally, it is known that the norm of a $D$-dimensional random vector of sub-Gaussain components, distributed as $\bm{\beta}_i \sim \mathcal{N}(0,P/D)$ will take values close to $\sqrt{P}$ (the radius of the $D$-sphere in question) with high probability. This can be seen in the concentration inequality: $\mathbb{P}\left(| \|\bm{\beta}\|_2 - \sqrt{P} |\geq t\right)\leq 2\exp\left(-c t^2 / K^4\right)$, for all $t>0$, where $c$ is a constant, and $K=\max_i \|\bm{\beta}\|_{\psi_2}$, where $\|\cdot\|_{\psi_2}$ is the sub-Gaussian norm (see Equation (3.3) in \cite{vershynin2018high}). Thus the domain of the $D$-ball integral will contain most $\bm{\beta}$ vectors with a high probability, and thus the volume of the $D$-balls calculated hold with high probability.


\subsection{Channel Capacity Redundancy Theorems}\label{app:sec_baron_clark}
Here, we make clear that channel capacity (which manifests as the supremum of a mutual information), is related to the MDL generalization term via the min-max KL risk. These ideas are well-known, and stated by Clarke \& Barron in $\S 2$ and $\S 5.1$ of \cite{clarke1994jeffreys}, and by Rissanen in $\S 1$ and $\S 4$ of \cite{rissanen1996fisher}.   

\begin{theorem}[Redundancy-Capacity Theorem \cite{clarke1994jeffreys}] \label{thm:red_cap_dual} The maximum channel capacity equals to the min-max KL divergence, 
\begin{align}
    \overbrace{\sup_{p(\bm{\beta})} \text{KL}\left(p(y|\beta)p(\beta) \| p(\beta) p(y)\right)}^{\text{\tiny{Channel Capacity}}} = \inf_{p\in\mathcal{P}} \sup_{\beta\in\mathcal{B}} \left(\text{KL}\left(p(y|\beta) \| p(y) \right) \right),
\end{align}
\end{theorem}

\begin{theorem}[Redundancy-Generalization Theorem \cite{clarke1994jeffreys} \label{thm:red_gen_barron}]
In the infinite limit of $N$, the min-max KL risk approaches the MDL generalization estimator:
\begin{equation}
    \lim_{N\rightarrow \infty} \left[\inf_{p\in\mathcal{P}} \sup_{\beta\in\mathcal{B}} \left( \text{KL}\left(p(y|\beta) \| p(y) \right) \right) - \overbrace{\frac{D}{2}\log\left(\frac{N}{2\pi }\right) - \log\left(\int_{\mathcal{B}}\sqrt{\det\left(\mathcal{I}(\beta)\right)}d\beta \right)}^{\text{\tiny{MDL Generalization}}} \right]=0
\end{equation}
\end{theorem}
 
 Theorem \ref{thm:red_cap_dual} relates to Theorem \ref{thm:red_gen_barron} through the appearance of the $\inf_{p\in\mathcal{P}} \sup_{\beta\in\mathcal{B}} \left(\text{KL}\left(p(y|\beta) \| p(y) \right) \right)$ term (the min-max KL risk). Specifically, these theorems make clear of an asymptotic relationship between the channel capacity, and the out-of-sample test error.


 \subsection{Proof of Theorem \ref{thm:chan_cap}} \label{app:sec_chan_cap_proof}
 
In this subsection we provide a proof for $\mathcal{C}$ as it appears in Theorem \ref{thm:chan_cap}. We note that there exists a proof for a similar looking system, involving, $X_{ij},\varepsilon_i\sim \mathcal{CN}(0,1)$ being circularly complex in Telatar's seminal paper~\cite{telatar1999capacity}. However, since properties over complex random matrix spaces do not necessarily lend themselves to real random matrices, we must work to provide a quick alternate proof for the case of $X_{ij}\sim \mathcal{N}(0,1)$, and with more general noise term, $\varepsilon\sim\mathcal{N}(0,\sigma^2)$ - which are the required assumptions for this particular paper, and for the problem of isotropic linear regression in general. To prove Theorem \ref{thm:chan_cap}, we call upon Theorem 9.2.1 in Pinsker \cite{pinsker1964information}, which relates the mutual information between real Gaussian random vectors in a convenient form via the logarithm of the determinant. This is expressed as Theorem \ref{thm:Pinsker} below. 
\begin{theorem}[Pinsker's Mutual Information \cite{pinsker1964information}]\label{thm:Pinsker}
Let $\xi = (\xi_1, ..., \xi_n)$,  $\eta = (\eta_1, .., \eta_m)=(\xi_{n+1}, ...,\xi_{n+m}), $ and $(\xi,\eta) = (\xi_1, ..., \xi_{m+n})$ be mean zero Gaussian random vectors taking values in some corresponding measurable Cartesian product spaces, $X = X_1 \times ... \times X_n$, $Y = Y_1 \times ... \times Y_m = X_{n+1} \times ... \times X_{m+n}$, and where $Z = X\times Y$. Then,
\begin{align}
    \mathcal{I}(\xi;\eta) = \frac{1}{2}\log\left(\frac{\det(A_{\xi})\cdot \det(A_{\eta})}{\det(A_{(\xi,\eta)})}\right),
\end{align}
where $A_{\xi} = \mathbb{E}[\xi\xi^{\intercal}]$, and $A_{\eta} = \mathbb{E}[\eta\eta^{\intercal}]$.
\end{theorem}
We now move onto proving Theorem \ref{thm:chan_cap}, and begin by relating the notation used by Pinsker, to that used within this paper.  
\begin{proof}
Consider $\bm{\beta}=\xi$, and $\bm{y} = \eta$, which implies that $A_\xi = \mathbb{E}[\bm{\beta} \bm{\beta}^{\intercal}]$ and $A_{\eta}= \mathbb{E}[\bm{yy}^{\intercal}]$. Moreover we initially consider a fixed (that is, non random) value of $\bm{X} = X$ to perform calculations. Through maximum entropy arguments, we consider $\mathbb{E}[\bm{\beta} \bm{\beta}^{\intercal}] = \frac{P}{D}I_{D}$, which is based on the power constraint of $\mathbb{E}[\bm{\beta}^{\intercal}\bm{\beta}]\leq P$, coupled with the fact that each $\bm{\beta}_i\sim\mathcal{N}(0,P/D \cdot I)$. For details on the choice of this prior in regards to maximum entropy, see the end of Appendix \ref{app:reg:iso_vol_lem}, and in addition see $\S$ 4.1 of Telatar as to why this is a natural choice for: $\sup\mathcal{I}(\bm{y};\bm{\beta})$, in regards to concavity of $\log\circ\det$ \cite{telatar1999capacity}. Now consider: $A_{\eta} = \mathbb{E}[\bm{yy}^{\intercal}] = \frac{P}{D} X X^{\intercal} + \sigma^2 I_{N}$, via linearity of expectation. In order to calculate $A_{(\xi,\eta)}$ we consider block matrices as follows:
\begin{align*}
  A_{(\xi,\eta)} &= \mathbb{E}\left[\begin{bmatrix} \bm{\beta} \\ \bm{y} \end{bmatrix}\begin{bmatrix}\bm{\beta}^{\intercal} \hspace{2mm} \bm{y}^{\intercal} \end{bmatrix}\right]\\
  &= \mathbb{E}\left[\begin{bmatrix} \bm{\beta} \bm{\beta}^{\intercal} & \bm{\beta} \bm{y}^{\intercal} \quad\\ \bm{y}\bm{\beta}^{\intercal} & \bm{y} \bm{y}^{\intercal} \end{bmatrix}\right] \\
  &= \begin{bmatrix} A_{\xi} & \mathbb{E}\left[\bm{\beta} \bm{y}^{\intercal}\right] \\ \mathbb{E}\left[\bm{y}\bm{\beta}^{\intercal}\right] & A_{\eta} \end{bmatrix},
\end{align*}
It can be shown that, $\mathbb{E}[\bm{\beta} \bm{y}^{\intercal}] = \mathbb{E}[\bm{\beta} \bm{\beta}^{\intercal} X^{\intercal} + \bm{\beta} \bm{\varepsilon}^{\intercal}] = \frac{P}{D}X^{\intercal}$, and note that the top right, and bottom left elements are transposes of one another. Thus by expanding the determinant of the above block matrix system, we obtain:
\begin{align*}
    A_{(\xi,\eta)} &= \det\left(A_{\xi}\right)\det\left( A_{\eta} - \frac{P}{d}X X^{\intercal} \right) \\ 
    &= \det\left(A_{\xi}\right)\det\left(\sigma^2 I_N \right)
\end{align*}
\vspace{-0.2cm}
\begin{align}
    \Rightarrow \mathcal{I}(\xi;\eta) &= \frac{1}{2}\log\left(\frac{\det(A_{\xi})\det(\sigma^2 + \frac{P}{D}X X^{\intercal})}{\det(A_{\xi})\det(\sigma^2 I_n)}\right)\nonumber\\
     &= \frac{1}{2} \log \circ \det\left(I_N + \frac{\text{SNR}}{D}X X^{\intercal}\right)\label{eqn:Channel_cap_proof1}
\end{align}
We once again emphasise that  $\mathcal{I}(\bm{y};\bm{\beta})=\mathcal{I}(\xi;\eta)$ and that the determinants have originated from Theorem \ref{thm:Pinsker}. Moreover, the expression thus far has been calculated for a fixed $\bm{X}=X$; that is, more accurately we have an expression for  $\mathcal{I}(\xi;\eta)=\mathcal{I}(\bm{y}|\bm{X}=X;\bm{\beta})$, with the conditioning implicitly assumed, whereas we ultimately desire a form for $\mathcal{I}((\bm{y},\bm{X});\bm{\beta})$. Considering now random $\bm{X}$:
\begin{align}
    \mathcal{I}((\bm{y},\bm{X});\bm{\beta}) &= \mathbb{E}\left[\log\left(\frac{p(\bm{y},\bm{X},\bm{\beta})}{p(\bm{y},\bm{X})p(\bm{\beta})}\right) \right]\nonumber\\
    &=\mathbb{E}_{p(\bm{X})}\left[\mathbb{E}\left[\log\left(\frac{p(\bm{y},\bm{\beta}\mid \bm{X} = X)}{p(\bm{y}\mid\bm{X}=X)p(\bm{\beta})}\right) \right]\right]\nonumber \\
    &= \mathbb{E}_{p(\bm{X})}\left[\mathcal{I}(\bm{y}|\bm{X}=X;\bm{\beta})\right],\label{eqn:channel_cap_proof2}
\end{align}
wherein we may obtain the required form for $\mathcal{C}$ by considering Equations~\eqref{eqn:Channel_cap_proof1} and~\eqref{eqn:channel_cap_proof2}. 
\end{proof}


\subsection{Theorem \ref{thm:chan_cap_lower_upper_bound} - Channel Capacity Bounds}\label{app:sec:chan_cap_lower_upper_bound}

In this section we derive upper and lower bounds for the channel capacity, $\mathcal{C}$. We note that upper and lower bounds are required separately, for the cases of $D\leq N$, and $D>N$. This results in four bounds in total.

\begin{theorem}[Channel Capacity Bounds] \label{thm:chan_cap_lower_upper_bound}
The channel capacity $\mathcal{C}$, where $\Psi: \mathbb{R} \rightarrow \mathbb{R}$ is the digamma function, is bounded as follows:
\begin{align*}
\begin{rcases}
\frac{D}{2}\log\left(\frac{2\text{SNR}}{D}\right) + \frac{1}{2}\sum_{i=1}^D \Psi\left(\frac{N-i+1}{2}\right)\\ \\
\frac{N}{2}\log\left(\frac{2\text{SNR}}{D}\right) + \frac{1}{2}\sum_{i=1}^N \Psi\left(\frac{D-i+1}{2}\right) 
\end{rcases}
\leq \mathcal{C} \leq
\begin{cases}
\frac{D}{2}\log\left(\frac{N}{D}\text{SNR} + 1\right) &\text{   for  } D\leq N,\\ \\
\frac{N}{2} \log\left(\text{SNR} + 1\right)  &\text{   for  } D> N.
\end{cases}
\end{align*}
\end{theorem}
\begin{proof}
 \textbf{Lower Bound ($D\leq N$):}\\ \\
$A,B\succ 0 \Rightarrow \det(A+B)^{1/n} \geq \det(A)^{1/n} + \det{B}^{1/n}$ (Minkowski's inequality). We take $n=1$, fix an $\bm{X}=X$, and proceed as follows:
\begin{align*}
    \det\left(I_D + \frac{\text{SNR}}{D}X^{\intercal}X\right) &\geq 1 + \det\left(\frac{\text{SNR}}{D}X^{\intercal}X\right) \\
    & = 1 + \left(\frac{\text{SNR}}{D}\right)^D\det\left(X^{\intercal}X\right).
\end{align*}
Consider that for any $a$, $b$ such that $a > b$ we have $\log_m (a) > \log_m(b)$, for $m > 1$. Thus,
\begin{align*}
    \log\circ\det\left(I_D + \frac{\text{SNR}}{D}X^{\intercal}X\right) &\geq \log\left( 1 + \det\left(\frac{\text{SNR}}{D}X^{\intercal}X\right)\right) \\
    &> \log\circ \det\left(\frac{\text{SNR}}{D}X^{\intercal}X\right)\\
    &= D\log\left(\frac{\text{SNR}}{D}\right)+ \log\circ\det\left(X^{\intercal}X\right).
\end{align*}
\begin{align*}
   \Rightarrow \frac{1}{2}\mathbb{E}\left[\log \circ \det\left(I_D + \frac{\text{SNR}}{D}\bm{X}^{\intercal}\bm{X}\right)\right] \geq \frac{D}{2}\log\left(\frac{\text{SNR}}{D}\right) + \frac{1}{2}\mathbb{E}[\log\circ\det(\bm{X}^{\intercal}\bm{X})].
\end{align*}
Now, since each $\bm{X}_{ij}\sim\mathcal{N}(0,1)$, it follows that $\bm{X}^{\intercal}\bm{X}\sim \mathcal{W}(N,\bm{\Lambda})$; that is, $\bm{X}^{\intercal}\bm{X}$ is Wishart distributed, where $N$ is the degrees of freedom, and $\bm{\Lambda}$ is the Wishart scale matrix, which in this case is the identity matrix. For the Wishart distribution, there is a known expansion for the expected log-determinant as follows:
\begin{align*}
    \mathbb{E}[\log\circ \det(\bm{X}^{\intercal}\bm{X})] = \sum_{i=1}^D \Psi \left(\frac{N-i+1}{2} \right) + D\log 2
\end{align*}
with $\Psi(\cdot)$ being the standard digamma function \cite{bishop2006pattern}. Thus we arrive at:
\begin{align*}
    \mathcal{C}\geq \frac{D}{2}\log\left(\frac{2\text{SNR}}{D}\right) + \frac{1}{2}\sum^D \Psi\left(\frac{N-i+1}{2}\right),
\end{align*}

where the lower bound for $D>N$ follows similarly. 

\textbf{Upper Bound $(D\leq N)$:} 

From Jensen's inequality,
\begin{align*}
\mathcal{C} &\leq \frac{1}{2}\log \circ \det\left(I_D+\frac{\text{SNR}}{D}\mathbb{E}[\bm{X}^{\intercal}\bm{X}])\right) \\
&= \frac{1}{2}\log \circ \det\left(\text{diag}_D\left(1+\frac{\text{SNR}}{D}\mathbb{E}[\chi^2(N)]\right)\right) \\
&=\frac{1}{2}\log \prod^D \left(1+ \frac{\text{SNR}\cdot N}{D}\right) \\
&= \frac{D}{2} \log\left(\frac{N}{D}\text{SNR} + 1\right).
\end{align*}
The proof for the upper bound in $D>N$ follows similarly, where we would consider a $\chi^2(D)$ term instead. 
\end{proof}

We can see how the bound behaves in Figure~\ref{fig:SNR_bounds_chan_cap}. In particular, it seems to suggest that the lower bound is much tighter than the upper bound for high SNR. However, it exhibits a small ``dip'' at the transition point of $D\rightarrow N$. This is primarily due to the $\frac{N}{2}\log\left(2\text{SNR}/D\right)$ term in the lower bound, which for smaller SNR becomes negative if $\text{SNR} < D/2$. Figure~\ref{fig:SNR_bounds_chan_cap} naturally leads one to question if something may be said about the limiting nature of the upper and lower bounds as $D\rightarrow \infty$. We thus establish Corollary~\ref{cor:limit_bound}.

\begin{figure}[t]
\centering
\begin{subfigure}{0.45\textwidth}
\centering
\includegraphics[width=\textwidth, height=5cm]{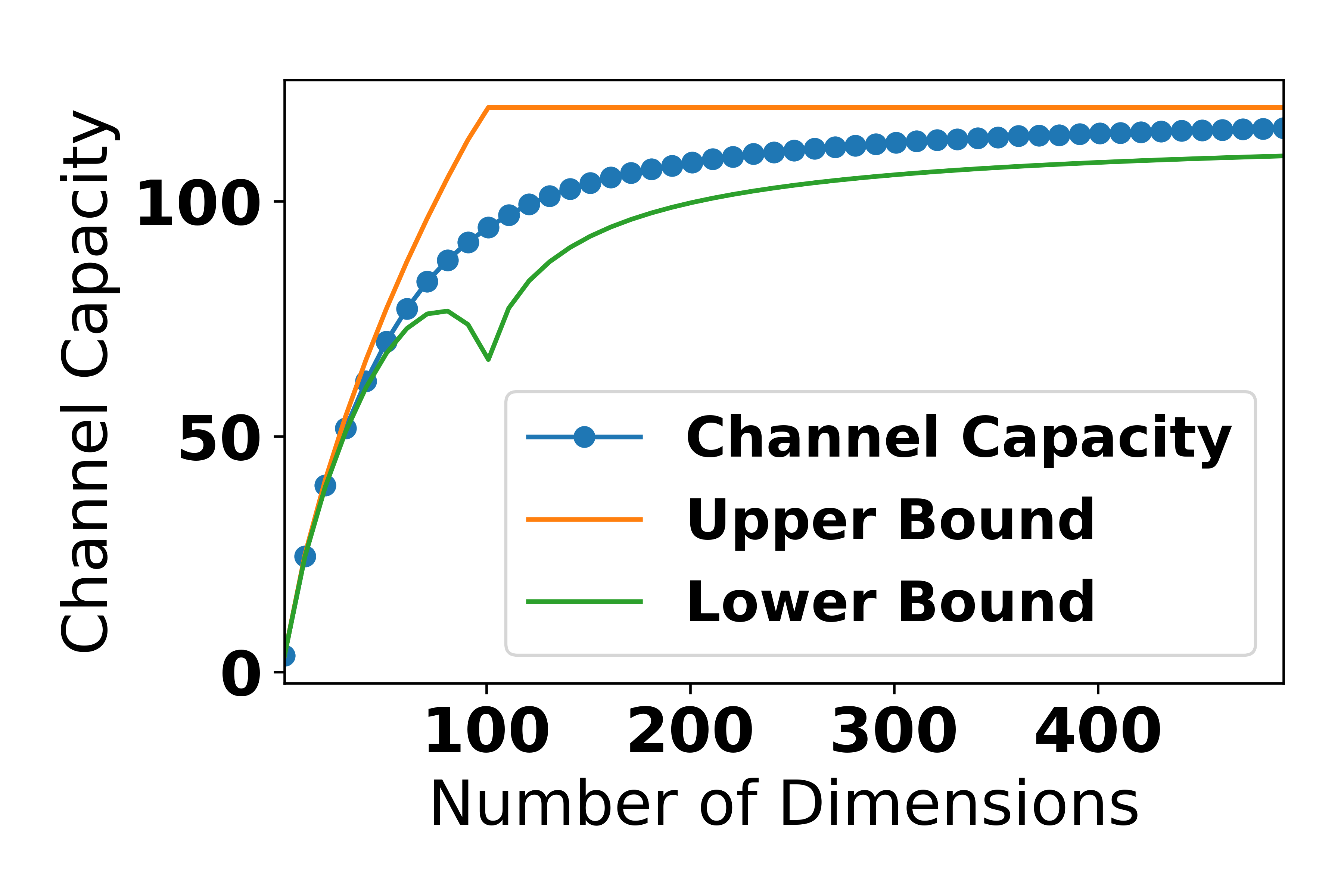} 
\caption{Channel capacity upper and lower bounds for isotropic linear regression if SNR = 10.}
\end{subfigure}%
\begin{subfigure}{0.45\textwidth}
\centering
\includegraphics[width=\textwidth, height=5cm]{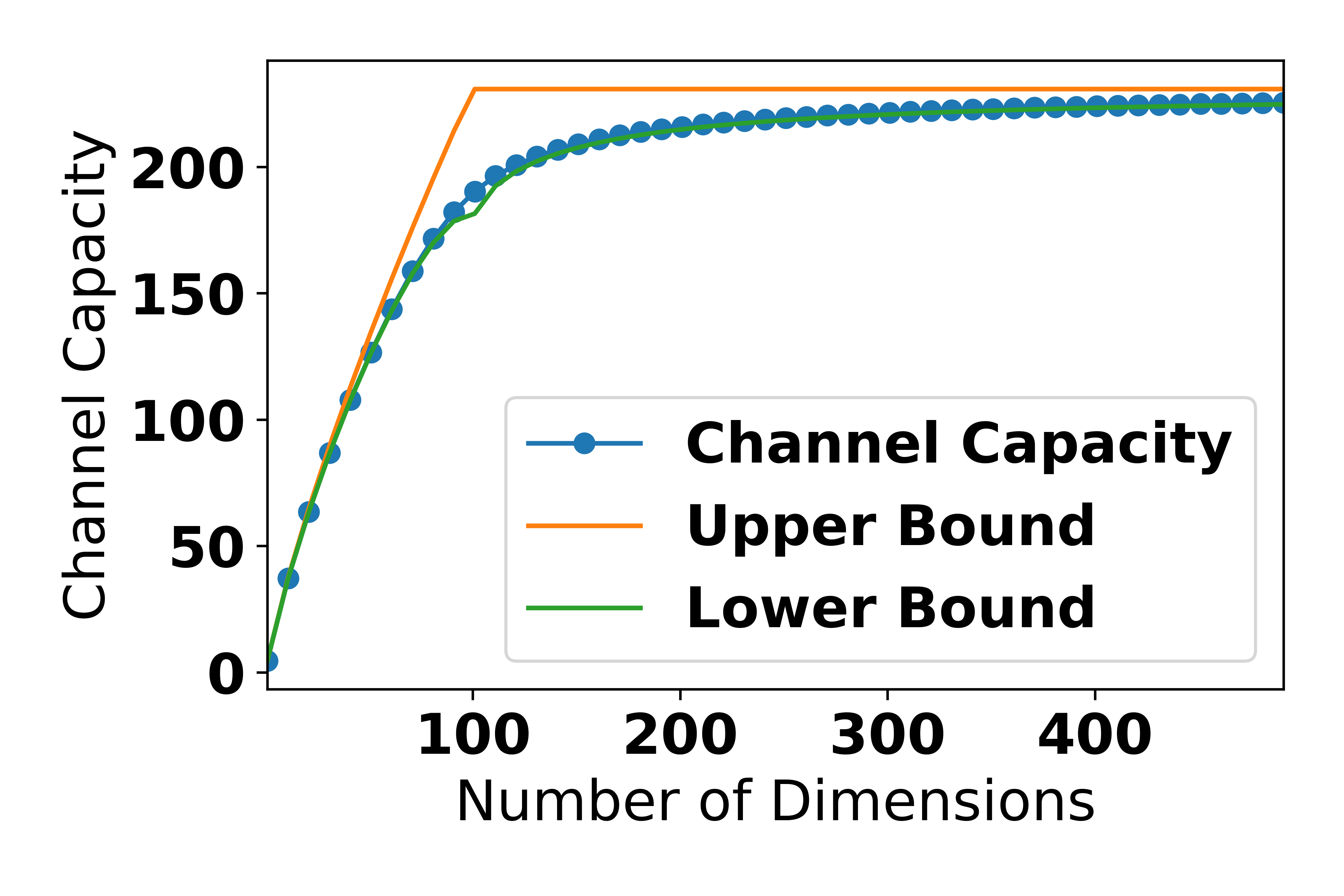}
\caption{Channel capacity upper and lower bounds for isotropic linear regression if SNR = 100.}
\end{subfigure}
\caption{Channel capacity bounds for different SNR. \label{fig:SNR_bounds_chan_cap}}
\end{figure}


\subsection{Corollary \ref{cor:limit_bound} - Convergence of Channel Capacity Limit}\label{app:sec:limit_bound}

In order to establish convergence we consider the following mild conditions: (1) A high SNR (SNR$\gg 1$), and (2) $ \Psi(x) \approx \log(x) - \frac{1}{2x}$, which is a common approximation to the digamma function, for large $x$. 

\begin{cor}[Channel Capacity Convergence] \label{cor:limit_bound}
For $D>N$, and under mild conditions we have, $
    \lim_{D\rightarrow\infty}\overline{\mathcal{C}} = \lim_{D\rightarrow\infty}\underline{\mathcal{C}}
    =\frac{N}{2}\log(\text{SNR}),
$ where $\overline{\mathcal{C}}$ and $\underline{\mathcal{C}}$ refer to the upper and lower capacity bounds which are established in Theorem~\ref{thm:chan_cap_lower_upper_bound}.

\end{cor}

\begin{proof}

\begin{align*}
   \underline{\mathcal{C}} &\approx \frac{N}{2} \log\left(\frac{2\text{SNR}}{D}\right) + \frac{1}{2}\sum_{i=1}^N \left[\log\left(\frac{D-i+1}{2}\right) - \frac{1}{2(D-i+1)}\right]\\
    &= \frac{1}{2}\sum_{i=1}^N \left[\log\left(\frac{2\text{SNR}(D-i+1)}{2D}\right) - \frac{1}{2(D-i+1)}\right]\\
     &= \frac{1}{2}\sum_{i=1}^N \left[\log\left(\frac{2\text{SNR}(1-i/D+1/D)}{2}\right) - \frac{1}{2(D-i+1)}\right] \\
   \Rightarrow \lim_{D\rightarrow \infty}  \underline{\mathcal{C}} &= \frac{1}{2}\sum_{i=1}^N \log(\text{SNR})\\
    &= \frac{N}{2}\log(\text{SNR})\\
    &= \lim_{D\rightarrow \infty}\overline{\mathcal{C}}, \quad \text{for SNR$\gg1$}
\end{align*}
\end{proof}





\subsection{Sphere Packing in AWGN} \label{app:sec:sphere_pack_AWGN}

For the additive white Gaussian noise (AWGN) case, we consider the system: $\bm{y} = \bm{x} + \bm{\varepsilon}$, where $\bm{x},\bm{y} \in \mathbb{R}^N$ and each $\bm{\varepsilon}_i\sim\mathcal{N}(0,\sigma^2)$. Under this system, each instantiation of $\bm{x}=x$ is approximately surrounded by a noise sphere of radius $\sqrt{\mathbb{E}[\bm{\varepsilon}^{\intercal}\bm{\varepsilon}]}=\sqrt{N\sigma^2}$. Moreover, we can define an approximate radius using $\bm{y}$ as: $\sqrt{\mathbb{E}[\bm{y}^{\intercal}\bm{y}]}=\sqrt{\mathbb{E}[\bm{x}^{\intercal}\bm{x}] + \mathbb{E}[\bm{\varepsilon}^{\intercal}\bm{\varepsilon}]} = \sqrt{NP + N\sigma^2}$, which describes the set of all possible \textit{codewords} decodeable when transmitting a signal $\bm{x}$. This is made geometrically clear in Figure \ref{fig:sphere_pack}. In this paper we extend this simple example to the ellipsoid case, which has an additional transformation factor of: $\sqrt{\det{\bm{XX}^{\intercal}}}$.

\begin{figure}[t]
\includegraphics[width=0.25\linewidth]{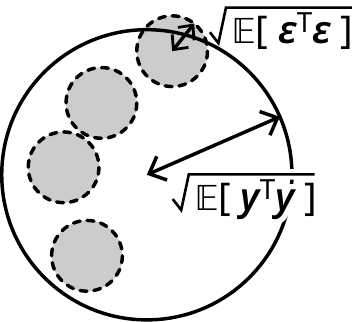}
\centering
\caption{Channel capacity is equivalent to sphere packing on a hypersphere for the case of AWGN. \label{fig:sphere_pack}}
\end{figure}

\section{Statistical Lattice Models}

\subsection{Reparameterzing the Dual Geometry of Lattices}\label{app:sec:reparam_poset}

Lattices are useful structures, as they allow one to efficiently encode information hierarchically. A geometry over lattice-type structures based on modelling higher order feature interactions via log probabilities has been derived in the work of Sugiyama et al.~\cite{sugiyama2016information}. 
The well-known log-linear model for binary variables in question is formulated as,
\begin{align*}
    \log \mathbb{P}(\bm{x}) = \sum_i \theta_{\{i\}} x_i + \sum_{i<j} \theta_{\{i, j\}}x_i x_j + \sum_{i<j<k} \theta_{\{i,j,k\}}x_i x_j x_k + \dots + \theta_{\{1,\dots,n\}}x_1 \dots x_n - \psi,
\end{align*}
where $\bm{x} \in \{0, 1\}^n$, each $\theta\in\mathbb{R}$ denotes the connection strength of a particular higher order interaction, each $x_i\in\{0,1\}$ denotes a binary valued variable which activates a particular connection strength, $\psi \in\mathbb{R}$ denotes the normalization constant for the probability model~\cite{Amari01,Agresti12}. Under this structure, $\mathbb{P}(\bm{x})$ is a member of the exponential family of distributions. If we define a particular instance of the partial ordering as $\bm{x} = (x_1, \dots, x_n) \le \bm{y} = (y_1, \dots, y_n)$, where $x_i \le y_i$ for all $i \in \{1, \dots, n\}$, and denote by $\Sigma(\bm{x})$ as the set of indices of ``1'' in $\bm{x}$, then when can instantiate a lattice, and can condense the representation of the above log-linear model as:
\begin{align*}
    \log \mathbb{P}(\bm{x}) = \sum_{\bm{s}} \delta(\bm{s}, \bm{x})\theta_{\Sigma(\bm{s})} = \sum_{\bm{s} \le \bm{x}} \theta_{\Sigma(\bm{s})}, \quad\text{where } \psi = -\theta_{\emptyset}.
\end{align*}
Hence the lattice is a natural representation of this hierarchical structure over the sample space of $\{0, 1\}^n$. Sugiyama et al.~\cite{sugiyama2016information} studied geometric structure of statistical lattice models and showed that distributions over not only $\{0, 1\}^n$ but any lattices belong to the exponential family. Note that posets are originally used in~\cite{sugiyama2016information}, which is a more general structure than lattices. Although we treat only lattices in this paper, most  interesting statistical models (such as Boltzmann machines) are lattices. We thus proceed in this direction as lattice structures entail a simple co-ordinate representation of the metric tensor as we have described in Theorem~\ref{thm:sugiyama_metric}.
An example of a lattice structure for $\{0, 1\}^3$ is illustrated in Figure~\ref{fig:lattice}.

\begin{figure}[t]
  \centering
  \includegraphics[width=.5\linewidth]{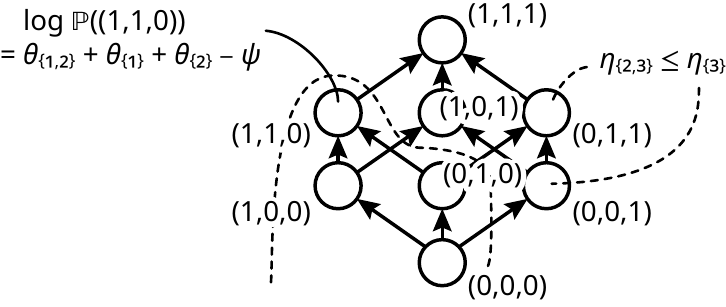}
  \caption{An example of a lattice structure for the binary domain $\{0, 1\}^3$. Each arrow denotes the partial order between elements in the lattice.}
  \label{fig:lattice}
\end{figure}

As Amari notes \cite{amari2007methods}, the exponential family of distributions induces a statistical manifold which possesses an interesting dualisitc structure. That is, two co-ordinate systems can be dually connected and allow one to generalize notions such as the Pythagoras theorem in Euclidean manifolds, to more general statistical manifolds. In particular, for the exponential family the first of these co-ordinate systems is given by $\bm{\theta} = (\theta_1 , ... , \theta_{D} )$ (as defined in the specified log-linear model of this subsection), and the second is given by $\bm{\eta} = (\eta_1, ... , \eta_{D} )$. Note that in the case of a binary log-linear model, we have $D = 2^n$ and 
\begin{align*}
    \eta_{\{i\}} &= \mathbb{E}[x_i] = \Pr(x_i = 1) \\
    \eta_{\{i,j\}} &= \mathbb{E}[x_i x_j] =  \Pr(x_i = 1, x_j = 1) \\
    \eta_{\{1, \dots, n\}} &= \mathbb{E}[x_1 ... x_n] =  \Pr(x_1 = 1, ... ,x_n = 1), 
\end{align*}
and $(\bm{\theta}, \bm{\eta})$ are explicitly dually connected via the Legendre transformation~\cite{amari2007methods,sugiyama2016information}. Note that since the $\eta$ co-ordinate system is defined probabilistically, and is thus constrained to be in $[0,1]^{D}$, which is convenient, and simplifies many calculations. Moreover, note that the $\eta$ co-ordinate system is built hierarchically, in the sense that it adheres to a partial ordering similar to the following \textit{example} structure:
\begin{align*}
    \eta_1 \geq \eta_2, \eta_3 \quad\eta_2 \geq \eta_4, \quad \eta_3 \geq \eta_5,  \quad \hdots
\end{align*}
This ordering is model specific and always uniquely determined from the lattice structure, and it is thus difficult to perform integrations over such arbitrary orderings. However, owing to the probabilistic nature of the $\eta$ co-ordinate system, it is possible to impose the following recursive re-parameterisation:
\begin{align*}
    \eta_1 &= \delta_1 \\
     \eta_2 &= \eta_1 +\delta_2 \\
      \eta_3 &= \eta_1 + \delta_3 \\
      \vdots \\
            \eta_D &= \eta_{D-1} + \delta_D 
\end{align*}
where each $\delta_i\in[0,1]$ for $1\leq i \leq D$, and $\sum_{i=1}^D \delta_i = 1$. Geometrically speaking, $\bm{\delta}=\{\delta_i\}_{i=1}^D$ represents points in a $D$-simplex. We can then proceed to formally encode the lattice ordering constraints via an additional zeta matrix, resulting in $\bm{\eta} = \mathcal{Z} \bm{\delta}$, where each $\mathcal{Z}_{ij}\in\{0,1\}$ is the value of zeta function $\zeta(q_i, q_j) = \mathbf{1}_{q_i \le q_j}$ for the corresponding elements $q_i$ and $q_j$ in the lattice. In other words, points from the $D$ simplex, $\bm{\delta}$, can be transformed into $\bm{\eta}$ co-ordinates for the poset manifold through a linear mapping. In order to re-express the volume integral via the $\bm{\delta}$ co-ordinates, it is necessary to calculate the determinant of the Jacobian transformation matrix between the co-ordinate systems. However this is trivially one, since $\mathcal{Z}$ is by construction upper triangular, resulting in $\det\left( \frac{\partial \bm{\eta}}{\partial \bm{\delta}} \right) = 1$. Therefore it is possible to calculate the log-volume integral as,
\begin{align}
     \log \left( \int_{\triangle_D} \sqrt{\det\left(\mathcal{G}(\bm{\delta})\right)\cdot \det\left( \frac{\partial \bm{\eta}}{\partial \bm{\delta}} \right)^2}d\bm{\delta} \right) =  \log \left( \int_{\triangle_D}  \sqrt{\det\left(\mathcal{G}(\bm{\delta})\right)}d\bm{\delta} \right),
\end{align}
where the coordinate transformation was performed using the square of the determinant, as the metric tensor is rank (0,2) - that is, it is a doubly covariant object, and we define the $D$-simplex as $\triangle_D$. In this form it is natural to re-express the volume integral via the expectation operator, where the expectation is taken with respect to a Dirichlet distribution over $\bm{\delta}$, as the Dirichlet distribution represents a pdf over the probability simplex. In other words we consider,
\begin{align}
     \log \left( \int_{\triangle_D} \sqrt{\det\left(\mathcal{G}(\bm{\delta})\right)}d\bm{\delta} \right) &= 
     \log \left( \int_{\triangle_D} \sqrt{\det\left(\mathcal{G}(\bm{\delta})\right)} \cdot \frac{w(\bm{\delta})}{w(\bm{\delta})} d\bm{\delta} \right) \nonumber \\
     &= \log\left( \mathbb{E}\left[\frac{\sqrt{\det\left(\mathcal{G}(\bm{\delta})\right)}}{w(\bm{\delta})}\right] \right), \label{eqn:E1}
\end{align}
where $
w(\bm{\delta}) 
= \frac {\prod _{i=1}^{D}\Gamma (\alpha _{i})}{\Gamma \left(\sum _{i=1}^{D}\alpha _{i}\right)}\prod _{i=1}^{D}x_{i}^{\alpha _{i}-1} \label{eqn:dir}
\coloneqq \text{Dir}(\bm{\delta} ; \bm{\alpha}) 
$, with ${\bm {\alpha }}=(\alpha _{1},\ldots ,\alpha _{D})$, and $\Gamma: \mathbb{R} \rightarrow \mathbb{R}$ being the standard Gamma function. Here, the choice of $\bm{\alpha}$ controls the manner in which sampling is performed over $\triangle_D$. We opt for a uniform exploration over the $D$-simplex, which equates to requiring that $\alpha_d = 1$ for all $d \in D$. Doing so means that we get, $w(\bm{\delta}) = \frac {\prod _{i=1}^{D}\Gamma (1)}{\Gamma \left(\sum _{i=1}^{D}1\right)} 
=\frac {1}{\Gamma(D)}$, where $\Gamma(D)=(D-1)!$. Ultimately, Equation~\eqref{eqn:E1} becomes,
\begin{equation}
  \log\left( \mathbb{E}\left[\frac{\sqrt{\det\left(\mathcal{G}(\bm{\delta})\right)}}{w(\bm{\delta})}\right] \right) = \log\left( \mathbb{E}\left[\sqrt{\det\left(\mathcal{G}(\bm{\delta})\right)}\right] \right) - \log\Gamma(D), 
\end{equation}
implying that bounding the volume can be equivalently achieved by bounding the behaviour of the term $\log\left( \mathbb{E}\left[\sqrt{\det(\mathcal{G}(\bm{\delta}))}\right] \right)$, and then appending the $\log\Gamma(D)$ term. The upper and lower bounds on this volume integral, are shown in Appendix \ref{app:sec:poset_bounds_proof}.

\subsection{Proof of Theorem \ref{thm:poset_vol_bound}}  \label{app:sec:poset_bounds_proof}

In this subsection we proceed to find lower and upper bounds for the $\log V$ in the case of the prescribed lattice geometry, by exploiting the re-parameterization of the $\eta$ co-ordinate system. 

\begin{proof}
\textbf{Volume Upper Bound:}

From Hadamarad's inequality: $|\text{det}(A)| \leq \prod_{i=1}^D A_{ii}$, for $A\in\mathbb{R}^{D\times D}$. Thus:
\begin{alignat*}{2}
  &\quad&  |\text{det}(\mathcal{G(\bm{\delta})})| &= \text{det}\left(\mathcal{G(\bm{\delta})}\right) \qquad (\mathcal{G}(\bm{\delta}) \succ 0)\\
 &\quad&   &\leq \prod_{i= 1}^D \mathcal{G}_{ii}(\bm{\delta}), \\
     \Rightarrow&& \sqrt{\text{det}(\mathcal{G}(\bm{\delta}))}                           &\leq \sqrt{\prod_{i= 1}^D \mathcal{G}_{ii}(\bm{\delta})} \\
 \iff&      &   \mathbb{E}\left[\sqrt{\text{det}(\mathcal{G}(\bm{\delta}))}\right]  &\leq \mathbb{E}\left[\sqrt{\prod_{i= 1}^D \mathcal{G}_{ii}(\bm{\delta})}\right]\\
 \iff&      &   \log\left(\mathbb{E}\left[\sqrt{\text{det}(\mathcal{G}(\bm{\delta}))}\right]\right)  &\leq \log\left(\mathbb{E}\left[\sqrt{\prod_{i= 1}^D \mathcal{G}_{ii}(\bm{\delta})}\right]\right)\\
  \iff&      &   \log\left(\mathbb{E}\left[\sqrt{\text{det}(\mathcal{G}(\bm{\delta}))}\right]\right) - \log\Gamma(D)  &\leq \log\left(\mathbb{E}\left[\sqrt{\prod_{i= 1}^D  \mathcal{G}_{ii}(\bm{\delta})}\right]\right) - \log\Gamma(D) \\
  \iff&      &   \log V  &\leq \log\left(\frac{\mathbb{E}\left[\sqrt{\prod_{i= 1}^D \mathcal{G}_{ii}(\bm{\delta})}\right]}{\Gamma(D)} \right),
\end{alignat*}
\\

\textbf{Volume Lower Bound:}
\begin{align}
    \log\left( \mathbb{E}\left[\sqrt{\det\left(\mathcal{G}(\bm{\delta})\right)}\right] \right) &\geq  \mathbb{E}\left[\log\left(\sqrt{\det\left(\mathcal{G}(\bm{\delta})\right)}\right)\right] \label{eqn:lower0} \\
    &= \frac{1}{2}\mathbb{E}\left[\log\circ\det\left(\mathcal{G}(\bm{\delta})\right)\right], \label{eqn:lower1}
\end{align}
where Inequality~\eqref{eqn:lower0} is Jensen's inequality. Moreover, $\mathcal{G}\succ 0\Rightarrow\exists\mathcal{M}$ s.t. $\mathcal{G} = \mathcal{M} \mathcal{M}^{\intercal}$, where $\mathcal{M}$ is a triangular matrix (Cholesky decomposition). Thus,
\begin{align}
    \det\left(\mathcal{G}\right) &= \det\left(\mathcal{M} \mathcal{M}^{\intercal}\right) \nonumber \\
                  &=  \det\left(\mathcal{M}\right)\cdot \det\left( \mathcal{M}^{\intercal}\right) \nonumber\\
               &=  \det\left(\mathcal{M}\right)\cdot \det\left( \mathcal{M}\right)\nonumber \\
              &=  \det\left(\mathcal{M}\right)^2 \nonumber\\
              &=  \left(\prod_{i=1}^D\mathcal{M}_{ii} \right)^2, \label{eqn:det_proof_line}
\end{align}
\begin{alignat*}{2}    &\Rightarrow&\frac{1}{2}\mathbb{E}\left[\log\circ\det\left(\mathcal{G}(\bm{\delta})\right)\right] &= \frac{1}{2}\mathbb{E}\left[\log\left(\prod_{i=1}^D\mathcal{M}_{ii}(\bm{\delta}) \right)^2\right] \\
   &\quad& &= \mathbb{E}\left[\sum_{i=1}^D\log \left(\mathcal{M}_{ii}(\bm{\delta})\right)\right], \\
   \Rightarrow  &&  \log\left( \mathbb{E}\left[\sqrt{\det\left(\mathcal{G}(\bm{\delta})\right)}\right] \right)  &\geq \mathbb{E}\left[\sum_{i=1}^D\log \left(\mathcal{M}_{ii}(\bm{\delta})\right)\right] \\
 \iff&      &  \log\left( \mathbb{E}\left[\sqrt{\det\left(\mathcal{G}(\bm{\delta})\right)}\right] \right) - \log\Gamma(D)&\geq \mathbb{E}\left[\sum_{i=1}^D\log \left(\mathcal{M}_{ii}(\bm{\delta})\right)\right] - \log\Gamma(D) \\
 \iff&      &   \log V  &\geq  \mathbb{E}\left[\sum_{i=1}^D \log\left(\mathcal{M}_{ii}(\bm{\delta})\right)\right] + \log\left(\frac{1}{\Gamma(D)}\right).
\end{alignat*}
\end{proof}

This decomposition provides insight into how $\log V$ operates, because of its split into the distinct additive components of \textit{richness}, and \textit{distinguishability}. Such ideas were shown similarly for Theorem \ref{lem:iso_vol_lem} in the case of isotropic linear regression, and in Theorem \ref{thm:perceptron_volume} for the stochastic perceptron unit.


\subsection{Proof of Remark \ref{rem:volume_poset_limit_0}}\label{app:sec:poset_vol_0}
Here we show that as $D\rightarrow \infty$, $V\rightarrow 0$ which implies that $\log V\rightarrow -\infty$. This is an important indicator in the increase of generalization performance, as sufficiently large $D$ can therefore overpower the $\mathcal{O}(D)$ model complexity term, present in traditional AIC and BIC.

\begin{proof}
As this represents a volume integral a trivial lower bound is zero. It follows that
\begin{gather*}
    0 \leq V \leq \mathbb{E}\left[ \frac{ \sqrt{\prod_{i= 1}^D \mathcal{G}_{ii}(\bm{\delta})}}{\Gamma(D)}\right] \\
 \Rightarrow  \lim_{D\rightarrow \infty} 0 \leq \lim_{D\rightarrow \infty} V \leq \lim_{D\rightarrow \infty} \mathbb{E}\left[ \frac{ \sqrt{\prod_{i= 1}^D \mathcal{G}_{ii}(\bm{\delta})}}{(D-1)!}\right].
\end{gather*}
Since each $\mathcal{G}_{ii}(\bm{\delta})\in[0,1]$, the factorial in the denominator strongly dominates, so that:
\begin{align*}
    \lim_{D\rightarrow \infty} \mathbb{E}\left[ \frac{ \sqrt{\prod_{i= 1}^D \mathcal{G}_{ii}(\bm{\delta})}}{(D-1)!}\right] = 0.
\end{align*}
Thus from via an application of squeeze theorem we see that, 
\begin{gather*}
    0 \leq \lim_{D\rightarrow \infty} V \leq 0, \quad \\ \Rightarrow \lim_{D\rightarrow \infty } V = 0.
\end{gather*}
\end{proof}

\section{Stochastic Perceptron Unit} \label{app:sec:perceptron_title}

We extend the current thesis to the classic stochastic perceptron unit. In particular, we consider a simple perceptron unit of the form: $y=f(\bm{w\cdot x}) + \varepsilon$, where $\varepsilon \sim \mathcal{N}(0,\sigma^2)$, and $f(\cdot)$ is the sigmoid non-linearity. Moreover, we follow Amari's parameterization from \cite{amari1997information}, which considers the input data distributed as, $\bm{x}\sim \mathcal{N}(0,I)$, and defines,  $w=\|\bm{w}\|$, and $\xi\sim\mathcal{N}(0,1)$, so that we can deal with the term: $\bm{w\cdot x}=w\xi\sim\mathcal{N}(0,w^2)$. Based on this set-up, we obtain Theorem \ref{thm:perceptron_volume}. 
\begin{theorem}[Log-Volume of Perceptron]\label{thm:perceptron_volume}
     \begin{align}
     \log V=    \footnotesize{\overbrace{\log \int_{\Omega_{\bm{w}}} \sqrt{\mathbb{E}_{\xi}\left[\xi^2 f'(w\xi) \right]\mathbb{E}_{\xi}\left[f'(w\xi) \right]^{D-1}} d w}^{\text{\tiny{``Richness''}}}} + \footnotesize{\overbrace{\log \frac{\mathbb{B}^D(1)}{\sigma^D}}^{\text{\tiny{``Distinguishability''}}}}.
     \end{align}
\end{theorem}
Theorem \ref{thm:perceptron_volume} makes clear that $\log V$ for the stochastic perceptron unit similarly decomposes into distinct \textit{distinguisability} and \textit{richness} components. The distinguishability term closely parallels that already derived for isotropic linear regression (Thereom \ref{lem:iso_vol_lem}) as it is driven by the volume of a $D$-ball, and scaled against noise. However, the \textit{richness} term here strongly encodes the model architecture - as opposed to linear regression, as it contains the derivatives of the non-linearity function, in relation to the norm of the input weights, with the distribution over $\bm{x}$. From observing this $\log V$ decomposition, it would appear beneficial for the network architecture to be modeled with non-linearities whose derivatives are well-constrained, so that as $D\rightarrow\infty$, we can better guarantee that $V\rightarrow 0$, implying good generalizability. As it turns out, the popular sigmoid, ReLU, and tanh activation functions do satisfy this property, as common to all we have: $|f'(\cdot)| \leq 1$.  However, the perceptron unit has a strongly singular geometry, which does pose a marked difficulty when extending its geometry to deeper, and more complex network architectures. On this note, a possible solution may lie in the work of Sun \& Nielsen~\cite{sun2019lightlike}. In particular, they have tried to circumvent the pathologies of singular semi-Riemannian geometries via a clever use of lightlike manifold structures. Doing so, they observe that under certain conditions for deep neural network (DNN) architectures, a Balasubramanian Occam's razor-like term may be developed as: $ -\log p (\bm{x}) \approx -\log(\hat{\mathcal{L}})  + \frac{D}{2}\log N - \mathcal{O}(DN^2)$ (Equation (21) in \cite{sun2019lightlike}), wherein the higher order term is developed based on a combined understanding of the empirical FIM, and $\log V$. At large $D$, and for fixed $N$, this term dominates, which they theorize may imply the strong compression, and generalization properties for DNNs; a conclusion which runs parallel to the thesis of this paper. 

\subsection{Stochastic Perceptron Metric Tensor}

We consider the following stochastic perceptron unit, 
\begin{align*}
   p(\bm{x},y;\bm{w},\sigma^2) = \frac{1}{\sqrt{2\pi\sigma^2}}\exp\left( -\frac{1}{2\sigma^2}(y-f(\bm{w}\cdot\bm{x}))^2\right)p(\bm{x}),
\end{align*}
 and note that its metric tensor has been derived previously by Amari for the case of $\sigma^2=1$. This is made clear in  Theorem~\ref{thm:percep_metric_tensor}. Moreover, we follow Amari's notation from \cite{amari1997information}, and consider the input data distributed as, $\bm{x}\sim \mathcal{N}(0,I)$, and  write: $w=\|\bm{w}\|$, and $\xi\sim\mathcal{N}(0,1)$, so that we can go on to define: $\bm{w\cdot x}=w\xi\sim\mathcal{N}(0,w^2)$. 

\begin{theorem}[Perceptron Metric Tensor \cite{amari1997information}] \label{thm:percep_metric_tensor}
The metric tensor for the stochastic perceptron model is given as:
\begin{align*}
    \mathcal{G}(w) = c_1 (w) I + (c_2 (w) - c_1 (w)) \mathbf{e_w e_w}^{\intercal},
\end{align*}
such that
\begin{align*}
    c_1 (w) &= \frac{1}{\sqrt{2\pi}}\int f(w\xi)(1-f(w\xi))\exp\left(-\frac{1}{2}\xi^2\right)d\xi,\\
    c_2 (w) &= \frac{1}{\sqrt{2\pi}}\int f(w\xi)(1-f(w\xi))\xi^2\exp\left(-\frac{1}{2}\xi^2\right)d\xi
\end{align*}
where $w = \|\bm{w}\|_2$, $\bm{e}_{\bm{w}} = \frac{\bm{w}}{w}$, and $f:\mathbb{R}\rightarrow \mathbb{R}$ is the sigmoid activation function.
\end{theorem}

Its proof can be inferred from another one of Amari's works~\cite{amari1998natural}, where instead it is performed with respect to the tanh activation function. Based on this, we note that Theorem \ref{thm:percep_metric_tensor} can be written into an expectation form, with the inclusion of a generic noise term $\varepsilon \sim \mathcal{N}(0,\sigma^2)$. This is shown in Proposition \ref{prop:sampling_metric_tensor}.

\begin{prop}[Metric Tensor Expectation]\label{prop:sampling_metric_tensor} The metric tensor can be condensed into expectation form as:
\begin{align}
   \mathcal{G}(w) = \frac{1}{\sigma^2}\left(\mathbb{E}_{\xi}\left[f'(w\xi) \right] + \mathbf{e_w e_w}^{\intercal}\mathbb{E}_{\xi}\left[f'(w\xi) (\xi^2-1)\right]\right).
\end{align}
\end{prop}

\begin{proof}
Take the expectation with respect to the Gaussian measure on $\xi$. 

The derivative, and noise terms arise from realizing that, if:  
\begin{align*}
    \log p(y,\bm{x};\bm{w},\sigma^2) = -\frac{1}{2\sigma^2}(y-f(\bm{w}\cdot\bm{x}))^2  + \log p(\bm{x}) + C, 
\end{align*}
Then $\frac{\partial p}{\partial w_i}=-\frac{1}{\sigma^2}(y-f(\bm{w}\cdot\bm{x}))\cdot f'(\bm{w}\cdot\bm{x})\cdot\bm{x}_i$. 
\begin{align*}
   \Rightarrow \mathcal{G}_{ij}(w)&=\mathbb{E}\left[\frac{\partial}{\partial w_i}\log p\cdot \frac{\partial}{\partial w_j}\log p\right]\\
    &=\frac{1}{\sigma^2}\mathbb{E}\left[f'(\bm{w\cdot x})^2 \bm{x}_i \bm{x}_j\right],
\end{align*}
where we can notice: $y-f(\bm{w}\cdot\bm{x})=\varepsilon$, which implies $\mathbb{E}[\varepsilon^2]=\sigma^2$. 

Changing the $\bm{x}_i\bm{x}_j$ term into the required $\xi^2$ which appears, requires calculating the inner products: $\langle\bm{w},\bm{w}\rangle_{\mathcal{G}}$, and $\langle\bm{v},\bm{v}\rangle_{\mathcal{G}}$, where $\bm{v}\perp\bm{w}$. This is described in greater depth from Amari's work \cite{amari1998natural}, and is in fact used to help prove its appearance in Theorem \ref{thm:percep_metric_tensor}, in $c_2(w)$ (or lack thereof in $c_1(w)$).
\end{proof}

This is a form which allows the integral to be evaluated simply via sampling methods. Based on this, it is possible to derive the formula for the volume of the information manifold induced by the stochastic perceptron unit, where we write it in a form which makes clear of the contributions of the \textit{richness} term, and the \textit{distinguishability} term.

\subsection{Proof of Theorem \ref{thm:perceptron_volume}} \label{app:sec:percp_vol_proof}
Amari originally states the required form of $V$ shown via Equation~\eqref{eqn:amarivol} without proof, in \cite{amari1997information}. However, in here we proceed to provide a short proof, and extend it slightly so that it is expressed in a form which makes clear of the existence of \textit{richness}, and \textit{distinguisability} terms. 

\begin{proof}
Applying the matrix-determinant Lemma over the metric tensor in Theorem \ref{thm:percep_metric_tensor}, we have
\begin{align}
    \det(\mathcal{G}(w)) &= \det(c_1 (w) I ) \left(1-(c_2 (w) - c_1 (w) )\cdot \mathbf{e_w}^{\intercal} (c_1 (w) I )^{-1} \mathbf{e_w} \right) \nonumber \\
    &= c_1 (w)^D \cdot \left(1-\frac{c_2 (w) - c_1 (w)}{c_1 (w)}\right) \nonumber\\
    &= c_1 (w)^{D-1} c_2 (w)\\
  \Rightarrow  V &= \int_{\Omega_{\bm{w}}} \sqrt{c_2(w)c_1(w)^{D-1}} d\bm{w}.\label{eqn:amarivol}
\end{align}
Since integration is performed over a $D$-ball of radius $w$, Amari writes: $d\bm{w}=\mathbb{B}^D (1)w^{D-1}dw$, which allows us to write, 
\begin{align*}
    \log V = \log \mathbb{B}^D(1) + \log \int_{\Omega_{\bm{w}}} \sqrt{c_2(w)c_1(w)^{D-1}} d w.
\end{align*}

Noticing that $c_1(w) = \frac{1}{\sigma^2}\mathbb{E}_{\xi}\left[f'(w\xi) \right]$, and $c_2(w) = \frac{1}{\sigma^2}\mathbb{E}_{\xi}\left[\xi^2 f'(w\xi) \right]$:

\begin{align*}
    \log V &= \log \mathbb{B}^D(1) + \log \left(\frac{1}{\sigma^2}\right)^{D/2}\int_{\Omega_{\bm{w}}} \sqrt{\mathbb{E}_{\xi}\left[\xi^2 f'(w\xi) \right]\mathbb{E}_{\xi}\left[f'(w\xi) \right]^{D-1}} d w \\
    &= \log \frac{\mathbb{B}^D(1)}{\sigma^D} + \log \int_{\Omega_{\bm{w}}} \sqrt{\mathbb{E}_{\xi}\left[\xi^2 f'(w\xi) \right]\mathbb{E}_{\xi}\left[f'(w\xi) \right]^{D-1}} d w.
\end{align*}

\end{proof}

\bibliographystyle{plain}
\bibliography{Bibliography}

\begin{thebibliography}{10}

\bibitem{Ackley85}
David~H. Ackley, Geoffrey~E. Hinton, and Terrence~J. Sejnowski.
\newblock A learning algorithm for {B}oltzmann machines.
\newblock {\em Cognitive Science}, 9(1):147--169, 1985.

\bibitem{Agresti12}
Alan Agresti.
\newblock {\em Categorical Data Analysis}.
\newblock Wiley, 3 edition, 2012.

\bibitem{akaike1973information}
Hirotogu Akaike.
\newblock Information theory and an extension of the maximum likelihood
  principle.
\newblock In B.~N. Petrov and F.~Caski, editors, {\em Proceedings of the 2nd
  International Symposium on Information Theory}, pages 267--281, 1973.

\bibitem{amari1997information}
Shun-ichi Amari.
\newblock Information geometry of neural networks---an overview---.
\newblock In {\em Mathematics of Neural Networks}, pages 15--23. Springer,
  1997.

\bibitem{amari1998natural}
Shun-Ichi Amari.
\newblock Natural gradient works efficiently in learning.
\newblock {\em Neural Computation}, 10(2):251--276, 1998.

\bibitem{Amari01}
Shun-ichi Amari.
\newblock Information geometry on hierarchy of probability distributions.
\newblock {\em IEEE Transactions on Information Theory}, 47(5):1701--1711,
  2001.

\bibitem{amari2016information}
Shun-ichi Amari.
\newblock {\em Information Geometry and Its Applications}, volume 194.
\newblock Springer, 2016.

\bibitem{amari2007methods}
Shun-ichi Amari and Hiroshi Nagaoka.
\newblock {\em Methods of information geometry}, volume 191.
\newblock American Mathematical Soc., 2007.

\bibitem{ba2020generalization}
Jimmy Ba, Murat Erdogdu, Taiji Suzuki, Denny Wu, and Tianzong Zhang.
\newblock Generalization of two-layer neural networks: An asymptotic viewpoint.
\newblock In {\em International Conference on Learning Representations}, 2020.

\bibitem{balasubramanian1996geometric}
Vijay Balasubramanian.
\newblock A geometric formulation of occam's razor for inference of parametric
  distributions.
\newblock {\em arXiv:adap-org/9601001}, 1996.

\bibitem{balasubramanian2005mdl}
Vijay Balasubramanian.
\newblock {MDL}, {B}ayesian inference, and the geometry of the space of
  probability distributions.
\newblock {\em Advances in Minimum Description Length: Theory and
  Applications}, pages 81--98, 2005.

\bibitem{Barron98}
Andrew~R. Barron, Jorma Rissanen, and Bin Yu.
\newblock The minimum description length principle in coding and modeling.
\newblock {\em IEEE Transactions on Information Theory}, 44(6):2743--2760,
  1998.

\bibitem{bartlett2019benign}
Peter~L Bartlett, Philip~M Long, G{\'a}bor Lugosi, and Alexander Tsigler.
\newblock Benign overfitting in linear regression.
\newblock {\em arXiv:1906.11300}, 2019.

\bibitem{belkin2019reconciling}
Mikhail Belkin, Daniel Hsu, Siyuan Ma, and Soumik Mandal.
\newblock Reconciling modern machine-learning practice and the classical
  bias--variance trade-off.
\newblock {\em Proceedings of the National Academy of Sciences},
  116(32):15849--15854, 2019.

\bibitem{belkin2018understand}
Mikhail Belkin, Siyuan Ma, and Soumik Mandal.
\newblock To understand deep learning we need to understand kernel learning.
\newblock {\em arXiv:1802.01396}, 2018.

\bibitem{bishop2006pattern}
Christopher~M Bishop.
\newblock {\em Pattern Recognition and Machine Learning}.
\newblock Springer, 2006.

\bibitem{clarke1994jeffreys}
Bertrand~S Clarke and Andrew~R Barron.
\newblock Jeffreys' prior is asymptotically least favorable under entropy risk.
\newblock {\em Journal of Statistical Planning and Inference}, 41(1):37--60,
  1994.

\bibitem{Cover06}
Thomas~M. Cover and Joy~A. Thomas.
\newblock {\em Elements of Information Theory}.
\newblock John Wiley \& Sons, 2006.

\bibitem{d2020double}
St{\'e}phane d'Ascoli, Maria Refinetti, Giulio Biroli, and Florent Krzakala.
\newblock Double trouble in double descent: Bias and variance(s) in the lazy
  regime.
\newblock {\em arXiv:2003.01054}, 2020.

\bibitem{Davey02}
Brian~A. Davey and Hilary~A. Priestley.
\newblock {\em Introduction to Lattices and Order}.
\newblock Cambridge University Press, 2 edition, 2002.

\bibitem{geiger2020scaling}
Mario Geiger, Arthur Jacot, Stefano Spigler, Franck Gabriel, Levent Sagun,
  St{\'e}phane d’Ascoli, Giulio Biroli, Cl{\'e}ment Hongler, and Matthieu
  Wyart.
\newblock Scaling description of generalization with number of parameters in
  deep learning.
\newblock {\em Journal of Statistical Mechanics: Theory and Experiment},
  2020(2):023401, 2020.

\bibitem{geiger2019jamming}
Mario Geiger, Stefano Spigler, St{\'e}phane d'Ascoli, Levent Sagun, Marco
  Baity-Jesi, Giulio Biroli, and Matthieu Wyart.
\newblock Jamming transition as a paradigm to understand the loss landscape of
  deep neural networks.
\newblock {\em Physical Review E}, 100(1):012115, 2019.

\bibitem{ghorbani2019linearized}
Behrooz Ghorbani, Song Mei, Theodor Misiakiewicz, and Andrea Montanari.
\newblock Linearized two-layers neural networks in high dimension.
\newblock {\em arXiv:1904.12191}, 2019.

\bibitem{Gierz03}
Gerhard Gierz, Karl~H. Hofmann, Klaus Keimel, Jimmie~D. Lawson, Michael
  Mislove, and Dana~S. Scott.
\newblock {\em Continuous Lattices and Domains}.
\newblock Cambridge University Press, 2003.

\bibitem{grunwald2007minimum}
Peter~D Gr{\"u}nwald.
\newblock {\em The minimum description length principle}.
\newblock MIT press, 2007.

\bibitem{hastie2019surprises}
Trevor Hastie, Andrea Montanari, Saharon Rosset, and Ryan~J Tibshirani.
\newblock Surprises in high-dimensional ridgeless least squares interpolation.
\newblock {\em arXiv:1903.08560}, 2019.

\bibitem{hastie2009elements}
Trevor Hastie, Robert Tibshirani, and Jerome Friedman.
\newblock {\em The Elements of Statistical Learning: Data Mining, Inference,
  and Prediction}.
\newblock Springer, 2009.

\bibitem{jeffreys1946invariant}
Harold Jeffreys.
\newblock An invariant form for the prior probability in estimation problems.
\newblock {\em Proceedings of the Royal Society of London. Series A.
  Mathematical and Physical Sciences}, 186(1007):453--461, 1946.

\bibitem{myung2000counting}
In~Jae Myung, Vijay Balasubramanian, and Mark~A Pitt.
\newblock Counting probability distributions: Differential geometry and model
  selection.
\newblock {\em Proceedings of the National Academy of Sciences},
  97(21):11170--11175, 2000.

\bibitem{nakkiran2019deep}
Preetum Nakkiran, Gal Kaplun, Yamini Bansal, Tristan Yang, Boaz Barak, and Ilya
  Sutskever.
\newblock Deep double descent: Where bigger models and more data hurt.
\newblock In {\em International Conference on Learning Representations}, 2020.

\bibitem{nakkiran2020optimal}
Preetum Nakkiran, Prayaag Venkat, Sham Kakade, and Tengyu Ma.
\newblock Optimal regularization can mitigate double descent.
\newblock {\em arXiv preprint arXiv:2003.01897}, 2020.

\bibitem{pinsker1964information}
Mark~S Pinsker.
\newblock {\em Information and Information Stability of Random Variables and
  Processes}.
\newblock Holden-Day, 1964.

\bibitem{rao1992information}
C~Radhakrishna Rao.
\newblock Information and the accuracy attainable in the estimation of
  statistical parameters.
\newblock In {\em Breakthroughs in Statistics}, pages 235--247. Springer, 1992.

\bibitem{rissanen1996fisher}
Jorma~J Rissanen.
\newblock Fisher information and stochastic complexity.
\newblock {\em IEEE Transactions on Information Theory}, 42(1):40--47, 1996.

\bibitem{rissanen1997stochastic}
Jorma~J Rissanen.
\newblock Stochastic complexity in learning.
\newblock {\em Journal of Computer and System Sciences}, 55(1):89--95, 1997.

\bibitem{schwarz1978estimating}
Gideon Schwarz.
\newblock Estimating the dimension of a model.
\newblock {\em The Annals of Statistics}, 6(2):461--464, 1978.

\bibitem{seldin2009pac}
Yevgeny Seldin.
\newblock {\em A PAC-Bayesian Approach to Structure Learning}.
\newblock Hebrew University, 2009.

\bibitem{shannon1948mathematical}
Claude~E Shannon.
\newblock A mathematical theory of communication.
\newblock {\em Bell System Technical Journal}, 27(3):379--423, 1948.

\bibitem{spigler2018jamming}
Stefano Spigler, Mario Geiger, St{\'e}phane d'Ascoli, Levent Sagun, Giulio
  Biroli, and Matthieu Wyart.
\newblock A jamming transition from under-to over-parametrization affects loss
  landscape and generalization.
\newblock {\em arXiv:1810.09665}, 2018.

\bibitem{sugiyama2016information}
Mahito Sugiyama, Hiroyuki Nakahara, and Koji Tsuda.
\newblock Information decomposition on structured space.
\newblock In {\em 2016 IEEE International Symposium on Information Theory},
  pages 575--579, 2016.

\bibitem{sugiyama2017tensor}
Mahito Sugiyama, Hiroyuki Nakahara, and Koji Tsuda.
\newblock Tensor balancing on statistical manifold.
\newblock In {\em Proceedings of the 34th International Conference on Machine
  Learning}, volume~70, pages 3270--3279, 2017.

\bibitem{sugiyama2018legendre}
Mahito Sugiyama, Hiroyuki Nakahara, and Koji Tsuda.
\newblock Legendre decomposition for tensors.
\newblock In {\em Advances in Neural Information Processing Systems}, pages
  8811--8821, 2018.

\bibitem{sun2019lightlike}
Ke~Sun and Frank Nielsen.
\newblock Lightlike neuromanifolds, occam's razor and deep learning.
\newblock {\em arXiv:1905.11027}, 2019.

\bibitem{telatar1999capacity}
Emre Telatar.
\newblock Capacity of multi-antenna gaussian channels.
\newblock {\em European Transactions on Telecommunications}, 10(6):585--595,
  1999.

\bibitem{tse2005fundamentals}
David Tse and Pramod Viswanath.
\newblock {\em Fundamentals of Wireless Communication}.
\newblock Cambridge university press, 2005.

\bibitem{Tulino04}
Antonia~M. Tulino and Sergio Verd{\'u}.
\newblock Random matrix theory and wireless communications.
\newblock {\em Foundations and Trends{\textregistered} in Communications and
  Information Theory}, 1(1):1--182, 2004.

\bibitem{vapnik2013nature}
Vladimir Vapnik.
\newblock {\em The Nature of Statistical Learning Theory}.
\newblock Springer, 2013.

\bibitem{vershynin2018high}
Roman Vershynin.
\newblock {\em High-Dimensional Probability: An Introduction with Applications
  in Data Science}, volume~47.
\newblock Cambridge university press, 2018.

\bibitem{watanabe2009algebraic}
Sumio Watanabe.
\newblock {\em Algebraic Geometry and Statistical Learning Theory}, volume~25.
\newblock Cambridge University Press, 2009.

\bibitem{xu2017information}
Aolin Xu and Maxim Raginsky.
\newblock Information-theoretic analysis of generalization capability of
  learning algorithms.
\newblock In {\em Advances in Neural Information Processing Systems}, pages
  2524--2533, 2017.

\end{thebibliography}

\end{document}